\documentclass{article} %
\usepackage{arxiv,times}
\usepackage{mystyle}

\usepackage{url}
\usepackage{enumitem}
\usepackage{geometry}
\usepackage{graphicx}
\usepackage{booktabs}
\usepackage{wrapfig}
\makeatletter
\AddToHook{cmd/appendix/before}{\def\cref@section@alias{appendix}}
\AddToHook{cmd/appendix/before}{\def\cref@subsection@alias{appendix}}
\makeatother

\makeatletter
\newcommand{\labeltext}[2]{%
  \@bsphack
  \MakeLinkTarget*{#1}%
  \def\@currentlabel{#1}{\label{#2}}%
  \@esphack
}
\makeatother

\title{A Spectral-Grassmann Wasserstein metric for operator representations of dynamical systems}

\author{Thibaut Germain\\
   CMAP, Ecole polytechnique\\
   \texttt{thibaut.germain@polytechnique.edu}\\
   \And 
   Rémi Flamary\\
   CMAP, Ecole polytechnique\\
   \texttt{remi.flamary@polytechnique.edu}\\
   \AND 
   Vladimir R. Kostic\\
   Istituto Italiano di Tecnologia\\
   University of Novi Sad\\
   \texttt{vladimir.kostic@iit.it}\\
   \And
   Karim Lounici\\
   CMAP, Ecole polytechnique\\
   \texttt{karim.lounici@polytechnique.edu}\\
   }

\preprintcopy %

\usepackage{amsmath,amsfonts,bm}

\def\eqref#1{equation~\ref{#1}}

\def\1{\bm{1}}

\DeclareMathAlphabet{\mathsfit}{\encodingdefault}{\sfdefault}{m}{sl}
\SetMathAlphabet{\mathsfit}{bold}{\encodingdefault}{\sfdefault}{bx}{n}

\newcommand{\R}{\mathbb{R}}

\newcommand{\reg}{\gamma}

\DeclareMathOperator*{\dom}{\ensuremath{\text{\rm dom}}}

\newtheorem{theorem}{Theorem}
\newtheorem{lemma}{Lemma}

\newcommand{\scalarp}[1]{{\langle #1\rangle}}

\providecommand{\norm}[1]{\lVert#1\rVert}
\providecommand{\SVDr}[1]{[\![#1]\!]_r}
\providecommand{\SVDrr}[2]{[\![#1]\!]_{#2}}
\providecommand{\abs}[1]{\lvert#1\rvert}

\newcommand{\C}{\mathbb C}
\newcommand{\N}{\mathbb N}

\newcommand{\levec}{\widehat{u}}
\newcommand{\revec}{\widehat{v}}
\newcommand{\refun}{\psi}
\newcommand{\erefun}{\widehat{\psi}}
\newcommand{\lefun}{\xi}
\newcommand{\elefun}{\widehat{\xi}}

\newcommand{\eval}{\lambda}

\newcommand{\eeval}{\widehat{\eval}}
\newcommand{\gap}{\ensuremath{\text{\rm gap}}}

\newcommand{\fH}{\phi}
\newcommand{\fG}{\psi}
\newcommand{\HSr}{{\rm{B}}_r({\RKHS})}
\newcommand{\Cx}{C_x}
\newcommand{\Creg}{C_\reg}

\newcommand{\Cxy}{C_{xy}}

\newcommand{\Kx}{K}

\newcommand{\Ky}{L}
\newcommand{\Kyx}{M}

\newcommand{\im}{\pi} %

\newcommand{\Koop}[1]{{A_{\im_{#1}}}}  %
\newcommand{\HTO}[1]{T_{#1}}
\newcommand{\ETO}[1]{\widehat{T}_{#1}}
\newcommand{\HKoop}{T}

\newcommand{\HSProj}[2]{{Q^{#1}_{#2}}}

\newcommand{\TS}{S_\im}  %
\newcommand{\ES}{\widehat{S}} %
\newcommand{\TZ}{Z_{\im}}  %

\newcommand{\Estim}{G}  %

\newcommand{\ECx}{\widehat{C}_x } %
\newcommand{\ECy}{\widehat{D}} %
\newcommand{\ECxy}{\widehat{C}_{xy}}  %
 
\newcommand{\ECreg}{\widehat{C}_\reg}
\newcommand{\EZ}{\widehat{Z}} %

\newcommand{\RKoop}{G_\reg}  %

\newcommand{\EEstim}{\widehat{G}}  

\newcommand{\EE}{\ensuremath{\mathbb E}}

\newcommand{\Data}{\mathcal{D}}

\newcommand{\X}{\mcX} %

\newcommand{\RKHS}{\mcH} %

\newcommand{\Lii}{\mcL^2_\im(\X)}   %

\providecommand{\norm}[1]{\lVert#1\rVert}

\newcommand{\HS}[1]{{\rm{HS}}\left(#1\right)} %

\newcommand{\Risk}{\mathcal{R}} %
\newcommand{\ExRisk}{\mathcal{E}_{\rm HS}} %
\newcommand{\IrRisk}{\mathcal{R}_0} %
\newcommand{\ERisk}{\widehat{\mathcal{R}}} %
\newcommand{\hnorm}[1]{\lVert#1\rVert_{\rm{HS}}} %

\newcommand{\rpar}{\alpha}
\newcommand{\spar}{\beta}

\DeclareMathOperator*{\range}{\ensuremath{\text{\rm Im}}}

\newcommand{\mXt}[1]{\mu_{#1}}
\newcommand{\spF}{\mathcal{F}}
\newcommand{\spX}{\mathcal{X}}
\newcommand{\TO}{A}
\newcommand{\IG}{L}

\begin{document}

\maketitle

\begin{abstract}
The geometry of dynamical systems estimated from trajectory data is a major challenge for machine learning applications. Koopman and transfer operators provide a linear representation of nonlinear dynamics through their spectral decomposition, offering a natural framework for comparison. We propose a novel approach representing each system as a distribution of its joint operator eigenvalues and spectral projectors and defining a metric between systems leveraging optimal transport. The proposed metric is invariant to the sampling frequency of trajectories. It is also computationally efficient, supported by finite-sample convergence guarantees, and enables the computation of Fréchet means, providing interpolation between dynamical systems. Experiments on simulated and real-world datasets show that our approach consistently outperforms standard operator-based distances in machine learning applications, including dimensionality reduction and classification, and provides meaningful interpolation between dynamical systems.
\end{abstract}

\section{Introduction}

Dynamical systems are widely used across scientific and engineering disciplines to model state variables' evolution over time \citep{lasota2013chaos}. Nonlinear ordinary or partial differential equations typically govern these systems and may incorporate stochastic components \citep{meyn2012markov}. However, in many practical situations, analytical models are unavailable or intractable, motivating the use of data-driven approaches to infer the underlying dynamics from sampled trajectories. In this context, Koopman and transfer operator regressions have emerged as a powerful framework for learning and interpreting dynamical systems from data \citep{brunton2021modern}. Rather than directly modeling the evolution of state variables, these operators advance observables (scalar functions defined on the state space) by mapping each to its expected future value conditioned on the current state. Crucially, these operators are linear even when the underlying systems are not linear. Under suitable conditions, they admit a spectral decomposition that provides insight into the system's long-term behavior, stability, and modal structure \citep{mauroy2020koopman}. These properties have made the operator-centric framework particularly appealing for both theoretical analysis and practical applications across various domains, including chemistry for molecular kinetics explainability \citep{wu2017variational}, robotics for control \citep{bruder2020data}, and fluid dynamics for prediction \citep{lange2021fourier}.

{
{\bf Koopman and transfer operators for dynamical systems.~}
From a learning standpoint, Koopman and transfer operators provide a compact and structured representation of dynamical systems, making them well-suited for machine learning applications requiring system comparison, such as time series classification \citep{surana2020koopman} and dynamical graph clustering \citep{klus2023koopman}. However, in order to leverage these representations in standard statistical and machine learning pipelines, one must first define a meaningful metric between them. Unfortunately, despite recent advances in operator estimation \citep{colbrook2023residual,kostic2023sharp,Kostic-Sub2024, bevanda2023koopman}, the development of similarity measures between operator representations of dynamical systems remains relatively underexplored despite the growing need for interpretable metrics on dynamical systems in machine learning applications \citep{ishikawa2018metric}.

}

{\bf Comparing dynamical systems.} We succinctly review existing similarity measures on dynamical systems; a detailed account is given in \Cref{appendix: related_work}. The case of (stochastic) linear dynamical systems (LDSs) and linear state-space models was first addressed in the literature \citep{afsari2014distances}. While early metrics are theoretically sound and leverage the manifold structure of LDS spaces, they suffer from high computational cost, making them impractical in most machine learning settings \citep{hanzon1982riemannian,gray2009probability}. Originally designed for ARMA models, the Martin pseudo-metric \citep{martin2002metric} offers a practical alternative and has later been extended to general LDS spaces and inspired kernel-based variants \citep{chaudhry2013initial}. These measures have been generalized to nonlinear systems through the Koopman/transfer operator framework \citep{fujii2017koopman,ishikawa2018metric}. More recent work considers topological conjugacy, where similarities can be defined via alignment methods \citep{ostrow2023beyond,glaz2025efficient} or Optimal Transport (OT) between operator spectra \citep{redman2024identifying}. A related line of research studies Wasserstein-type metrics on functional spaces such as \citet{antonini2021geometry}, introducing OT between measures derived from the eigenvalues of normal operators.\\
The above approaches face key limitations. Norm-based measures and the Martin pseudo-metric are noise-sensitive and lack interpretability. OT-based similarities improve interpretability by comparing spectral geometry, but they are restricted to self-adjoint operators and define pseudo-metrics rather than metrics. As a result, no existing method combines theoretical soundness, robustness, and computational efficiency, and defining a principled metric for dynamical systems remains an open challenge.

{\bf Contributions.} %
   In \Cref{section: method} we introduce a \textbf{novel representation of transfer operators as distributions over eigenvalues and associated eigenspaces on the Grassmann manifold}. Building on this, we define a new metric between dynamical systems via optimal transport between these distributions. We prove that this metric is theoretically principled, computationally efficient, and robust under data-driven operator estimation. We further establish spectral learning rates under weaker assumptions for reduced-rank Koopman estimators, advancing the state of the art and providing finite-sample convergence guarantees for our metric. Leveraging this geometry, we propose an algorithm to compute Fréchet means of dynamical systems. Finally, in \Cref{section: experiments}, we empirically demonstrate the advantages of our approach over existing metrics and apply it to machine learning tasks and system interpolation.

\section{Background}
\label{sec:background}

{\bf Linear evolution operators.~}
Let $(X_{t})_{t\in\mathbb{T}}$ be the flow in some state space $\X$ whose governing laws are temporally invariant, where the time index $t$ can be either discrete ($\mathbb{T}=\N_0$) or continuous ($\mathbb{T}=[0,+\infty)$). While the flows of many important dynamical systems are nonlinear and possibly stochastic, under quite general assumptions they admit \emph{linear operator representations} on a suitably chosen space of real-valued functions $\spF \subset \R^{\spX}$, henceforth referred to as observable space. Namely, letting $t \in \mathbb{T}$, the \emph{transfer operator}, also known as \emph{Koopman operator} for deterministic systems, {$\TO_{t}\colon\spF\to\spF$ evolves an observable $f: \spX \rightarrow \mathbb{R}$} for time $t$ via conditional expectation
\begin{equation}\label{eq: Koopman operator}
[\TO_{t}( f )](x) 
:= \EE[ f(X_{t})\,\vert X_{0} = x],\;x\in\spX. 
\end{equation}
Clearly, since $\TO_t \TO_s = \TO_{t+s}$, in the discrete-time setting the process can be studied only through the transfer operator $\TO:=\TO_{1}$ of one unit of time, typically a second. On the other hand, when time is continuous, the process is characterized by the infinitesimal generator of the semigroup $(\TO_{t})_{t\geq 0}$, defined as 
$
\IG := \lim_{t\to 0^+} (\TO_{t} - \Id)/t
$
that is a differential operator with domain in $\spF$ that encodes the equations of motion and generate dynamics as $\TO_t = e^{\IG t}$, see~\cite{Lasota1994,ross1995stochastic}. 

{\bf Spectral decomposition.~} The utility of transfer operator representations stems from its linearity on a suitably chosen $\spF$ that is {\it invariant} space under the action of $\TO_{t}$, that is $\TO_{t} [\spF] \subseteq \spF$ for all $t$ (a property that we tacitly assumed above), and \textit{rich enough} to represent the flow of the process, i.e., it contains observables from which we can reconstruct all the relevant information of the state  (e.g. in the case of a stochastic system distribution $\mXt{t}$ at any time $t$ ). Namely, using the spectral theory of linear operators \cite{kato2013perturbation}, under suitable assumptions, one can spectrally decompose generator $\IG{=}\sum_{j\in J} (\lambda_j \, P_j +N_j) {+} P_c \IG$ into distinct complex scalars $\lambda_j{\in}\C$, called eigenvalues, forming point-spectrum and mutually commuting Riesz spectral projectors $P_j$ that satisfy satisfy equations $\IG P_j {=} \lambda_j P_j$ and $P_c {=} I {-} \sum_{j}P_j$, $P_J$ being of finite rank $m_j$ (geometric multiplicity), $N_j$ being nilpotent, and $j{\in} J$ being countably many. Assuming for simplicity that $\spF$ is a separable Hilbert space and $\IG$ is a non-defective operator with purely discrete spectrum, e.g. stable diffusion
processes, see~\cite{ross1995stochastic}, we have that
$\IG {= }\textstyle{\sum_{j \in \mathbb{N}}} \lambda_j \, g_j \otimes_{\spF} f_j$,
with $\IG f_j {=} \lambda_j f_j$, $\IG^* g_j {=}\overline{\lambda_j} g_j$, and $\langle f_j, g_j \rangle_{\spF} {=} \delta_{i,j}$,
where $(\lambda_j,f_j,g_j)_{j\in\mdN}$ are eigen-triplets consisting of an eigenvalue, left and right eigenfunction, respectively. This, in turn, allows one to decouple the evolution of an arbitrary observable $f{\in}\spF$
\begin{equation}\label{eq: modal decomposition}
\mdE[f(X_t) \ | \ X_0 = x_0] = [\TO_t f](x)=\textstyle{\sum_{j \in \mdN}} e^{\lambda_j t} \innerp{f_j}{g_j}_{\spF} f_j(x_0) = \textstyle{\sum_{j \in \mdN}} e^{\tau_j t} e^{ i 2\pi \omega_j t} m_j^f (x_0),
\end{equation}
into modes $m_j^f=\innerp{f_j}{g_j}_{\spF} f_j\colon\spX\to\R$ that evolve as scalar oscillators at timescales given by reciprocals of $\tau_j = \Re(\lambda_j)$ and frequencies $\omega_j = \Im(\lambda_j)\,/\,2\pi$ in Hz (assuming time in seconds). 

{\bf Learning transfer operators.~}
In machine learning applications, dynamical systems are only observed, and neither $\TO$ nor its domain, such as the space of square integrable functions w.r.t. the equilibrium measure, is known, providing a key challenge to learn them from data. 
The most popular algorithms~\citep{brunton2021modern} aim to learn the action of $\TO\colon \spF\,{\to}\,\spF$ on a predefined, possibly infinite dimensional, Reproducing Kernel Hilbert Space (RKHS), resulting in estimating the {\em restriction} of $\TO$ on $\RKHS\subseteq\spF$ by projection, that is $P_{\RKHS}\TO_{\vert_{\RKHS}}\colon\RKHS\to\RKHS$, typically via empirical risk minimization,
~\cite{kostic2022learning}. When $\RKHS$ is given by a universal reproducing kernel~\cite{Steinwart2008}, meaning it is dense in $\spF$, such techniques have strong spectral estimation guarantees \cite{kostic2023sharp}, can forecast well the states \cite{bevanda2023koopman,Alexander2020}, and evolve distributions of stochastic processes via kernel mean embeddings \cite{kostic2024consistent}. As an alternative, finite-dimensional $\RKHS$ spaces can be used with sieve methods \cite{kutz2016dynamic} or be learned from data in the form of rich neural representations~\cite{Liu2024}, that can be also trained to minimize the projection error $\Vert P_{\RKHS}^\bot\TO_{\vert_{\RKHS}}\Vert_{\spF\to\spF}$~\cite{Kostic-ICLR2024}. In these settings, a major limitation of the existing statistical learning guarantees is assuming well-specifiedness, i.e., the existence of an exact RKHS representation of $\TO$. A more realistic learning scenario requires that only the most relevant spectral part of $\TO$ lives in a suitable universal RKHS space.

{\bf Discrete optimal transport.}~
Optimal Transport (OT) is a well-defined framework to compare probability distributions, with many applications in machine learning~\citet{peyre2019computational}. In discrete OT, one seeks a transport plan mapping samples from a source distribution to those of a target distribution while minimizing a transportation cost.
Formally, consider $\mcZ_S = \{z_i \in \mcZ \ | \ i \in [k_S]\}$ and $\mcZ_T =
\{z'_i \in \mcZ \ | \ i \in [k_T]\}$ as the sets of source and target samples in
a space $\mcZ$. We associate to these sets the probability distributions $\mu_S
= \sum_{i \in [k_S]}a_i \delta_{z_i}$ and $\mu_T = \sum_{i \in [k_T]}
b_i\delta_{z'_i}$  with $(\mba,\mbb) \in \Delta^{k_S} \times \Delta^{k_T}$ and
$\Delta^n = \{ \mbp \in \mdR_+^n\ | \ \sum_{i \in [n]}p_i = 1\}$ the $n$-simplex. Let
$\mbC \in \mdR_+^{k_S \times k_T}$ be the cost matrix with $C_{ij} =
c(z_i,z'_j)$ being the transport cost between $z_i$ and $z'_j$ given by the cost
function $c$. The Monge-Kantorovich problem aims at identifying a coupling
matrix, also denoted as OT plan $\mbP^* \in \mdR_+^{k_S \times k_T}$, that is a solution of the
constrained linear problem:
\begin{equation}
   \label{eq: OT problem}
   \min\limits_{\mbP \in \Pi(\mu_S,\mu_T)} \innerp{\mbC}{\mbP}_F \quad \text{s.t} \quad \Pi(\mu_S,\mu_T) = \{ \mbP \in \mdR_+^{k_S \times k_T} \ | \ \mbP\mb1 = \mba,~ \mbP^\intercal\mb1 = \mbb \}~,
\end{equation}
where $\Pi(\mba,\mbb)$ is the set of joint-distributions over $\mcZ_S \times
\mcZ_T$ with marginals $\mba$ and $\mbb$. In what follows, we denote
$L_c(\mu_S,\mu_T)$ the application returning the optimal value of problem (\ref{eq:
OT problem}) where $c$ indicates the cost function. A fundamental property of OT
is that, under suitable conditions on the cost function, the Wasserstein
distance defined as $ W_p  (\mu,\nu) \triangleq (L_{d^p}(\mu,\nu))^\frac{1}{p}$ is a metric on the
space of probability measures \citet[Theorem 6.18]{villani2008optimal}.

\section{Spectral-Grassmann Optimal Transport (SGOT)}
\label{section: method}

\textbf{Problem setting and assumptions.}
Machine learning tasks on (stochastic) dynamical systems, such as comparing trajectories, identifying regimes, or clustering dynamics, require a discriminative and computationally efficient notion of distance between observed processes. We address this by representing each system through its associated Koopman/transfer operator and then introducing the SGOT metric to compare them. To that end, let us formalize the problem setting and main assumptions.\\
\textbf{\labeltext{(A1)}{enum:assumption1}}\hspace{3mm}\textbf{(A1, Dynamical
systems functional spaces and sampling)} Consider $N \in \mdN^*$
time homogeneous, Markovian dynamical systems defined on a common state space $\mcX$ and characterized by their generators $\IG_{k}\colon\dom(\IG_{k})\subset\spF_{k}{\to}\spF_{k}$ defined on the respective spaces $\spF_{k}$ of observables $\X{\to}\R$, $k{\in}[N]$. For every $k{\in}[N]$, let $\Data_{k}{=}\{(x_i^{k},y_i^{k})\}_{i\in[n_k]}$ be a dataset of observations of the of the $k$-th system, consisting of consecutive states separated by time-lag $\Delta t_k$. Notably, in the case of a single trajectory $y^k_i = x_{i+1}^k$.\\
Since data-driven methods can distinguish between systems only up to the temporal resolution at which the observations are made~\cite{zayed2018advances}, recalling \eqref{eq: modal decomposition}, the systems differing in spectral
components beyond the observable range of timescales $1/\tau_j$ and frequencies $\omega_j$ are undistinguishable from measurements. Therefore, we focus below on spectral projections of dynamics that can be learned from finite data.\\
\textbf{\labeltext{(A2)}{enum:assumption2}}\hspace{3mm}\textbf{(A2, Low rank)} For every $k{\in}[N]$ there exists $r_k{\in}\mdN$ such that $r_k$ eigenvalues of $\IG_{k}$ closest to the origin are separated from the rest of the spectrum, and let $P_{\leq r_k}\colon\spF_{k}{\to}\spF_{k}$ denote the corresponding spectral projector. \\
Recalling the case of dynamical systems sampled at equilibrium, i.e. $\spF_k=\mcL^2_{\pi_k}(\mcX)$ with $\pi_k$ being the invariant measure of the $k$-th system, a central conceptual difficulty in introducing distance between systems is that transfer operators for different systems naturally act on different spaces, and therefore cannot be compared directly. To resolve this, we restrict each operator to a common reproducing kernel Hilbert space $\RKHS$ that is included in the domain of all the transfer operators.\\ 
\textbf{\labeltext{(A3)}{enum:assumption3}}\hspace{3mm}\textbf{(A3, Common functional space)} Let $\mcH$ be a  separable RKHS associated with kernel $\kappa$, such that for all $k{\in}[N]$ it holds $\range(P_{\leq r_k}\IG_k)\subset \mcH \subset \mcF_{k}$. Hence, there exists representation $\HTO{k}{=}e^{(P_{\leq r_k}\IG_k)_{\vert_{\RKHS}}}\colon\RKHS{\to}\RKHS$ with spectral decomposition $\HTO{k}{=}\sum_{j\in[\ell_k]} e^{\lambda_j^{k}} \, \HSProj{k}{j}$ where $\ell_k$ is the number of distinct eigenvalues. %
\\
Assuming the sufficient richness of the RKHS space, the shared domain above ensures the operators are mathematically comparable via their restrictions. That this assumption is reasonable, it suffices to note that for typical Langevin dynamics, a universal Gaussian RBF RKHS with properly chosen landscale parameter contains a finite number of leading eigenfunctions of generators $\IG_{k}$ defined on $\mcL^2_{\pi_i}(\mcX)$ spaces weighted by Boltzmann distributions $\pi_k$. Furthermore, one can formally build a finite-dimensional space $\RKHS$ by choosing exactly the basis of such a generator's eigenfunctions; the complexity of the problem is then transferred to learning $\RKHS$. Beyond this stochastic case, one can similarly work in other domains, see e.g.~\cite{colbrook2025rigged,Alexander2020,bevanda2023koopman}.

{\bf Spectral Grassmanian Wasserstein metric.}~Since the spectral decomposition of a non-defective operator $\HTO{k}$ into its eigenvalues and spectral projectors is uniquely defined up to a permutation, any meaningful comparison approach based on operators' spectral decomposition must be invariant to permutations and change of basis in which spectral projectors are expressed. While discrete optimal transport naturally provides invariance to permutations through the minimizing coupling matrix \cref{eq: OT problem}, we need to design a ground metric that takes into account both spectral and subspace aspects to obtain a true OT metric. This is done below, where we define a Wassertein metric on the set of non-defective operators (complete proof in \Cref{appendix: proof sgot}).

\begin{thm}\label{thm:main_wasserstein_metric}
Let $\mcH$ be a separable $\mdC$-Hilbert space and $\mcS_r(\mcH)$ the set of non-defective operators with rank at most $r\in\mcD$. Let $(\mcG, d_\mcG)$ be the Grassmanian manifold of the space of Hilbert-Schmidt operators on $\RKHS$. Given $p{\in}\mdN^*$ and $\eta {\in} (0,1)$, let $\mu{\colon} S_r(\mcH){\to}\mcP_p(\mdC \times \mcG)$ and  $d_\eta{\colon} (\mdC \times \mcG)^2 {\to}\mdR_+$ be given by
\begin{equation}\label{eq: baseline metric}
\mu(\HKoop) {\triangleq} \textstyle{\sum_{j \in [\ell]} }\frac{m_j}{m_{tot}}\delta_{(\lambda_j,\mcV_j)}\quad\text{ and }\quad d_\eta[(\lambda',\mcV'),(\lambda', \mcV')]{\triangleq}\eta |\lambda-\lambda'| {+} (1{-}\eta) \,d_\mcG(\mcV,\mcV'),
\end{equation}
with $|\cdot|$ applied on polar coordinates $\lambda,\lambda'$, $m_{tot} = \sum_{i \in [\ell]}m_i$, $\mcV_j$ the $m_j$-dimensional vector space in ${\rm HS}(\RKHS,\RKHS)$ spanned by the rank one operators of the right/left eigenfunctions associated with the eigenvalue $e^{\lambda_j}$ of $\HKoop$ (same notation for $\HKoop'$). Then, $(\mcS_r(\mcH),d_{\mcS})$ is a metric space, where $d_{\mcS}{\colon}\mcS_r(\mcH){\to}\mdR_+$ is given by
\begin{equation}\label{eq: spectral metric on operator} 
      d_{\mcS}(\HKoop,\HKoop') =  W_{d_\eta,p}(\mu(\HKoop),\mu(\HKoop')).
   \end{equation}
\end{thm}

First, recalling \eqref{eq: modal decomposition} and \ref{enum:assumption1}, note that while typically in data-driven methods datasets are sampled at some frequency $\omega^{ref}_{k} = 1/\Delta t_k$ to estimate eigenvalues $e^{\lambda^{k}_i \Delta t_k}$ of transfer operators $\TO_{k}^{\Delta t_k}$, we build a metric using the difference in the generator eigenvalues. This is to compare Koopman modes' eigenvalues as physical quantities, since for the $k$-th system the observed time-scales are $\tau_j^k / \omega^{ref}_{k}$ and the oscillating frequencies $\omega_j^k / \omega^{ref}_{k}$. So, by re-normalizing eigenvalues, we can compare systems observed at different time-scales in the universal time units. 
Further, we remark that assuming non-defective operators is not a major bottleneck, since \Cref{thm:main_wasserstein_metric} can be extended to the space of general linear operators with rank at most $r$ by leveraging the Dunford-Jordan decomposition \citep{dunford1988linear}. In this case, the cost metric in $d_\eta$ compares the spectrum and subspaces of Jordan blocks.

{\bf Metric computation.}~In order to evaluate the SGOT metric, one needs to compute the cost matrix (see \cref{sec:background}), i.e., $d_\eta$  for each pair of spectrals. Following \ref{enum:assumption1}-\ref{enum:assumption3}, let $\widehat{\HKoop}$ be an operator estimated from samples $\{(x_i,y_i)\}_{i \in [n]}$ with a kernel based method. Suppose that $\widehat{\HKoop}$ admits $l$ eigenvalues, each with multiplicity $m_i$. Let ${\bs\beta_i, \bs\alpha_i} \in (\mdC^{n \times l_i})^2$ be the control parameters of the left/right eigenfunctions related to the $i^{\text{th}}$ eigenvalue and preprocessed to form an orthonormal basis. Let $\widehat{\HKoop}'$ be another estimated operator, and $\mbM_{\bs\epsilon} \triangleq \{k(\epsilon_i,\epsilon'_j)\}_{(i,j) \in [n]\times[n']}$ with $\bs\epsilon{\in}\{\mbx,\mby\}$, be the cross-kernel matrices. For $p=1$, the cost matrix $\mbC \in \mdR_+^{l \times l'}$ is given by:
\begin{equation}
C_{i,j} = \eta |\lambda_i-\lambda'_j| + (1-\eta)(m_i + m_j - 2 \Tr((\bs{\beta}_i^* \mathbf{M}_y\bs{\beta}_j)^*(\bs{\alpha}_i^*\mathbf{M}_x \bs{\alpha}_j)))^{\frac{1}{2}}~.
\end{equation}
With the rank $r\geq \max (l,l')$, the time complexity of $d_\mcS$ is in $O(n^2 r^2+r^3log(r))$ respectively due to the cost matrix computation and the OT solver (that is negligible for small $r$).
Consequently, $d_\mcS$ and the kernel metric computation are asymptotically equivalent, overcoming the usual computational drawbacks of OT-based methods relative to kernel ones.  If needed, both metrics can further benefit from standard kernel scaling techniques~\cite{Meanti2023}.

{\bf Statistical guarantees.}~In the following, we show how using RRR estimators yields unbiased estimation of the SGOT. To that end, consider $\ETO{k}=(\ECx^k{+}\gamma I)^{-\frac{1}{2}}\SVDrr{(\ECx^k {+}\gamma I)^{-\frac{1}{2}}\ECxy^k}{r_k}$, where $\ECx^k {=} \tfrac{1}{n_k}\sum_{i\in[n_k]} \kappa_{x_i^k}\otimes \kappa_{x_i^k}$, $\ECxy^k {=} \tfrac{1}{n_k}\sum_{i\in[n_k]} \kappa_{x^k_i}\otimes \kappa_{y^k_i}$, $\gamma>0$ and $\SVDr{\cdot}$ denoting best rank-$r$ approximation. As discussed above, one can efficiently compute $d_{\mcS}(\ETO{1},\ETO{2})$ so that the following holds.

\begin{restatable}{thm}{prpstatbound}
\label{thm:stat_bound}
Let \ref{enum:assumption1}-\ref{enum:assumption3} hold with $k\in[2]$, $\spF_k{=}\mcL^2_{\pi_k}(\spX)$ and $\kappa(x,x)<\infty$ a.s. for $x\sim\pi_k$. Let $\EE[\ECx^k]=\Cx^k$ and assume that for some $\alpha{\in}[1,2]$ and $\beta{\in}[0,1]$ it holds that $\Vert|[(\Cx^k)^\dagger]^{\frac{\alpha-1}{2}}\HTO{k}\Vert_{\RKHS{\to}\RKHS}{<}\infty$ and $\eval_i(\Cx^k)\,{\leq}\,\lesssim\,i^{-1/\spar}$ for $i\in\N$. Given $\delta \in (0,1)$, if $n$ is large enough and $\lambda_{r_k}{\lesssim}-\frac{\alpha\log n}{2(\alpha+\beta)}$, then w.p.a.l. $1 {-}\delta$ in the i.i.d. draw of samples $\Data_1$ and $\Data_2$ it holds
$| d_\mcS(\ETO{1},\ETO{2}) - d_\mcS(\HTO{1},\HTO{2}) | \lesssim  n^{-\frac{\rpar -1 }{2(\rpar+\spar)}}\ln(2\delta^{-1})$. 
\end{restatable}
\begin{proof}[Sketch of Proof]
To obtain this result, we needed to overcome the overly strong assumption of well-specifiedness of $\RKHS$ made in \cite{kostic2023sharp}, which significantly reduces the applicability of those bounds to estimate the distance between true generators $L_k$ with high probability. By carefully treating the approximation errors originating from rank reductions, $\widetilde{T}_k$ being population version of $\ETO{k}$, we obtain $\| \widetilde{T}_k  - \HTO{k}  \| \lesssim \reg^{\frac{\rpar-1}{2}}{+}e^{\lambda_{r_k}}$ under realistic assumption \ref{enum:assumption3}. Furthermore, we derive an upper bound on the operator norm $\| \widetilde{T}_k  - \ETO{k}\|_{\RKHS\to\RKHS} \lesssim \sqrt{\reg^{-\spar-1} n^{-1}} \log(\delta^{-1}) $ w.p.a.l. $1-\delta$. Balancing the two terms gives the bound on $\| \widetilde{T}_k  - \HTO{k}  \|$. Next, we apply standard polar analysis and Davis-Kahan perturbation analysis to derive the bound on $d_\eta(\widetilde{T}_k,\HTO{k})$. Finally, the stability property of the Wasserstein distance gives the final bound. Full proof is available in \cref{appendix: stat_bound}.
\end{proof}

{\bf Spectral Grassmann OT barycenter, parametric model and optimization.}
Computation of barycenters is fundamental for many unsupervised methods; it is known as the \emph{Fréchet mean problem} in metric spaces. It consists in identifying an element that minimizes a weighted sum of distances to the observations. Formally, given the importance weights $\bs\gamma {\in} \Delta^N$, assuming~\ref{enum:assumption1}-\ref{enum:assumption3}, and $p{=}2$ in \Cref{thm:main_wasserstein_metric}, we aim to solve:
\begin{equation}
   \label{eq:frechet_mean_problem}
   \argmin_{T \in \mcS_r(\mcH)} \sum_{k \in [N]} \gamma_i d_\mcS(\HKoop,\HTO{k})^2, 
\end{equation} 
By construction of $d_\mcS$, problem~\ref{eq:frechet_mean_problem} corresponds to a \emph{free-support Wassertein barycenter} estimation problem which aims at optimizing the support of the atoms parametrizing the barycenter, in our case, its spectral decomposition. State-of-the-art algorithms typically rely on a coordinate descent scheme, alternating between transport plan computation and measure optimization \citep{cuturi2014fast,claici2018stochastic}. \\
Whenever the RKHS $\mcH$ is infinite dimensional, the Fréchet mean problem (\cref{eq:frechet_mean_problem}) is intractable. 
So we restrain the optimization over a set of parametrized operators defined such that for any $\bs\theta \triangleq (\bs\lambda,\bs\alpha,\bs\beta,\mbx)$: 
\begin{equation}
   \label{eq: parametrized operator}
   T_{\bs\theta} : h \in \mcH \mapsto  \textstyle{\sum_{i \in [r]}} \lambda_i \innerp{ \kappa_{\mbx}\bs\alpha_i}{h}_\mcH \kappa_{\mbx}\bs\beta_i \in \mcH
\end{equation}
where $\bs\lambda \in \mdC^r$, $\mbx \in \mcX^n$ are state space control points, and $\bs\alpha,\bs\beta \in \mdC^{n \times r}$ control parameters acting on the representer functions $\kappa_\mbx = \{\kappa(.,x_j)\}_{j \in [n]}$ with $\kappa$ the kernel of $\mcH$, i.e. $\kappa_{\mbx}\bs\alpha_i \triangleq \sum_{j \in [n]}\kappa_{x_j}\alpha_{ji}$. While these operators are compact with rank at most $r$, further constraints on the control points and parameters are required to ensure a spectral decomposition (see \cref{eq: modal decomposition}). Together with the definition of discrete optimal transport (see \Cref{sec:background}), it leads to the constrained optimization problem:
\begin{equation}
   \label{eq: constrained spectral barycenter problem}
   \argmin\limits_{\bs\theta, \mbP} \sum_{i \in [N]} \gamma_i \innerp{\mbC_i(\bs\theta)}{\mbP_i}_F \quad \textrm{s.t.} \quad \left\{
   \begin{array}{ll}
    \bs \alpha^* \mbK \bs\beta = \mbI & \mbK = \{\kappa(x_i,x_j)\}_{(i,j) \in [n]^2} \\
    \bs\beta_j^* \mbK \bs\beta_j = 1,~ \forall j \in [r] & \mbP_i \in \Pi(\mu(T_{\bs{\theta}}),\mu(T_i)),~\forall i \in [N]
   \end{array}\right.
\end{equation}
where $\mbP = \{\mbP_i\}_{i \in [N]}$, $\widehat{\bs{T}} = \{\widehat{T}_i\}_{i \in [N]}$, such that $(\mbC_i(\bs\theta), \mbP_i)$ are the cost and transport matrices associated to the Wasserstein metric $d_\mcS$ defined in \Cref{thm:main_wasserstein_metric}, between the parametric operator $T_{\bs\theta}$ and $\widehat{T}_i$.\\
Following \citet{cuturi2014fast} and considering a differentiable kernel w.r.t the control points, we propose an inexact coordinate descent scheme with a cyclic update rule for optimizing problem~\ref{eq: constrained spectral barycenter problem}. Each cycle begins with the computation of the optimal transport plans, then the subsequent coordinate updates are performed with a few gradient descent steps and a closed-form projection scheme to enforce the constraints. In \Cref{appendix: barycenter} we provide more detail about the computational and theoretical aspects of the barycenters.

\section{Numerical experiments}
\label{section: experiments}

We now illustrate the benefits of our metric and barycenter through numerical experiments on dynamical systems. We first study the behavior of different similarity measures under various shifts and compare them on unsupervised and supervised machine learning tasks. Finally, we demonstrate the properties of operator barycenters using our proposed algorithm on two simulated examples.

\textbf{Compared similarity measures.} In addition to our proposed metric SGOT, we compare other OT-based similarities that focus solely on the eigenvalues (SOT)~\citep{redman2024identifying} or solely on the eigenspaces using a Grassmannian metric (GOT)~\citep{antonini2021geometry}. We also include metrics induced by the Hilbert–Schmidt and operator norms, as well as the Martin similarity~\citep{martin2002metric}.

\subsection{Comparison with other similarity measures}
\label{section: comparison behavior}
\begin{figure}[t]
    \centering \vspace{-8mm}
    \includegraphics[width=1.0\linewidth]{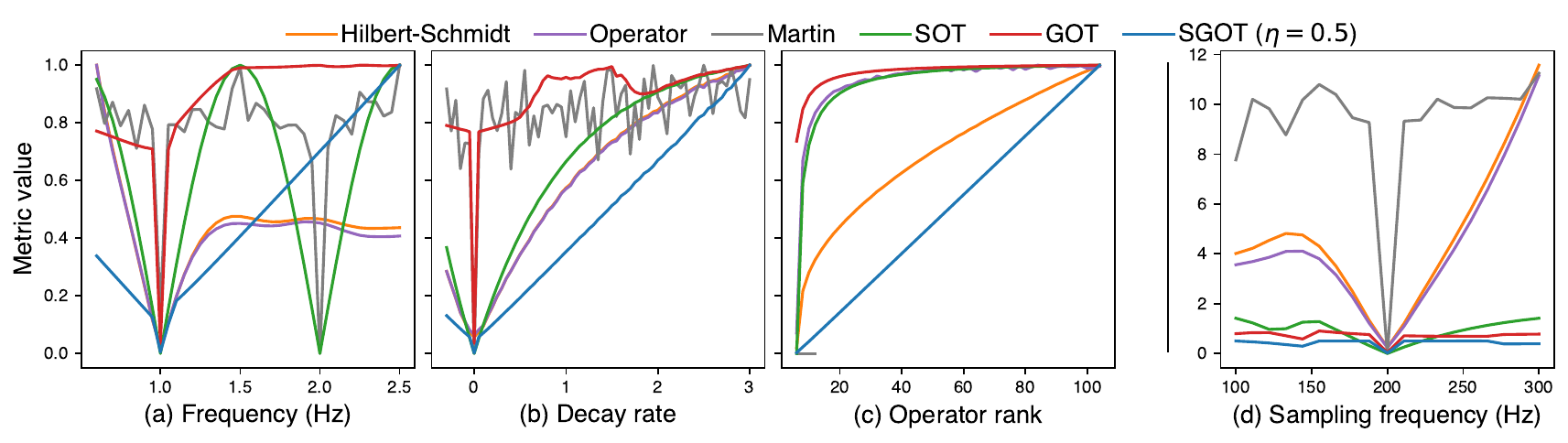}
    \vspace{-7mm}
    \caption{Similarity measures' behaviors under four scenarios of shifts of a linear oscillatory system: (a) frequency shift, (b) decay rate shift, (c) operator rank/subspace shift, (d) sampling frequency variation. In scenarios (a,b,c), metric values are normalized by their maximum.  }
    \vspace{-5mm}
    \label{fig: ablation study}
    
\end{figure}

{
\textbf{Simulated system and shifts.} First, we illustrate the behavior of different similarity measures between dynamical systems with regard to variations of the spectral decompositions of their Koopman operators.
 We consider a referent linear oscillatory system that is the sum of two simple harmonic oscillators with frequencies 0.5Hz and 1.0Hz, respectively, with a trajectory sampled at 200Hz. Considering the linear kernel, we compare the Koopman operator of the referent system with those of shifted systems according to four scenarios:  
    \textbf{(a) Frequency shift}, changes the 1Hz harmonic frequency.
    \textbf{(b) Decay rate shift}, changes the 1Hz harmonic decay rate.
    \textbf{(c) Subspace shift (rank)} gradually transforms the 1Hz sine
    wave into a 1Hz square wave signal using a  Fourier
    decomposition of a square wave signal with increasing order.
    \textbf{(d) Sampling frequency shift} where the system is sampled at different
    sampling frequencies instead of the reference 200Hz.
In each scenario, Koopman operators are estimated from sampled trajectories with the RRR method~\citep{kostic2022learning} with rank fixed to twice the number of harmonic oscillators. %

\textbf{Results \& interpretation.}
Values of the different metrics as a function of the shifts are shown in \Cref{fig: ablation study}. In scenarios (a,b,c), our metric SGOT grows linearly with the shifts almost everywhere. In contrast, other similarities tend to saturate quickly, and some even oscillate as shifts increase. In particular, OT-based competitors exhibit extreme behaviors: the pseudo-metric SOT oscillates in the frequency scenario, while GOT saturates fastest overall.
Likewise, the Hilbert-Schmidt and operator
metrics present both a saturating and an oscillating behavior, introducing many local minima. .
When changing the sampling frequency in scenario (d), only GOT and our metric SGOT are robust and remain low and almost constant. \Cref{appendix: comparison_exp} provides details and a sensitivity analysis of the $\eta$ parameter in SGOT. 
}

\subsection{Machine learning on dynamical systems}

\textbf{Experimental setup.}
We now illustrate and study the usability of our metric SGOT in machine learning
applications, both unsupervised and supervised, when sequential data are embedded by estimated operators governing their
dynamics.
In both experiments, time series are represented with Koopman operators estimated with the RRR method \citep{kostic2022learning} with a linear kernel. The experiments are run on 14 multivariate time series datasets from the UEA database \citep{ruiz2021great}. 

\begin{wrapfigure}{R}{0.4\linewidth}%
    \centering
    \vspace{-4mm}
    \includegraphics[width = \linewidth]{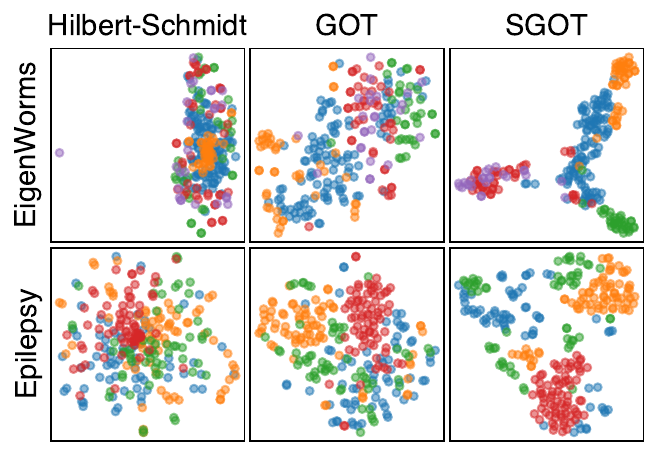}
    \vspace{-5mm}
    \caption{T-SNE embeddings. Datasets on rows, metrics on columns, classes in colors.}
    
    \label{fig: partial tsne} 
\end{wrapfigure}

\textbf{Dimensionality reduction.} We first explore the dimensionality reduction capabilities of the different similarity measures. For the 5 selected datasets and all similarities, the samples are embedded as 2D vector with the T-distributed Stochastic Neighbor Embedding (T-SNE) \cite{maaten2008visualizing}
method fitted on the cross-distance matrix estimated
with the similarity. \Cref{fig: partial tsne} illustrates the embeddings for the most discriminative similarities on datasets
\textit{EigenWorms} (motion) and \textit{Epilepsy} (biomedical). TSNE embedding for all 5 datasets and metrics are available in \Cref{appendix: ML additional results}, \Cref{fig:all tsne}. The Hilbert-Schmidt distance is too conservative, and no clusters or classes can be identified. For OT-based metrics, GOT better identifies classes; however, they do not form distinct clusters as is obtained with our metric SGOT.

\begin{figure}[t]
    \centering
    \vspace{-2mm}
    \includegraphics[width=.9\linewidth]{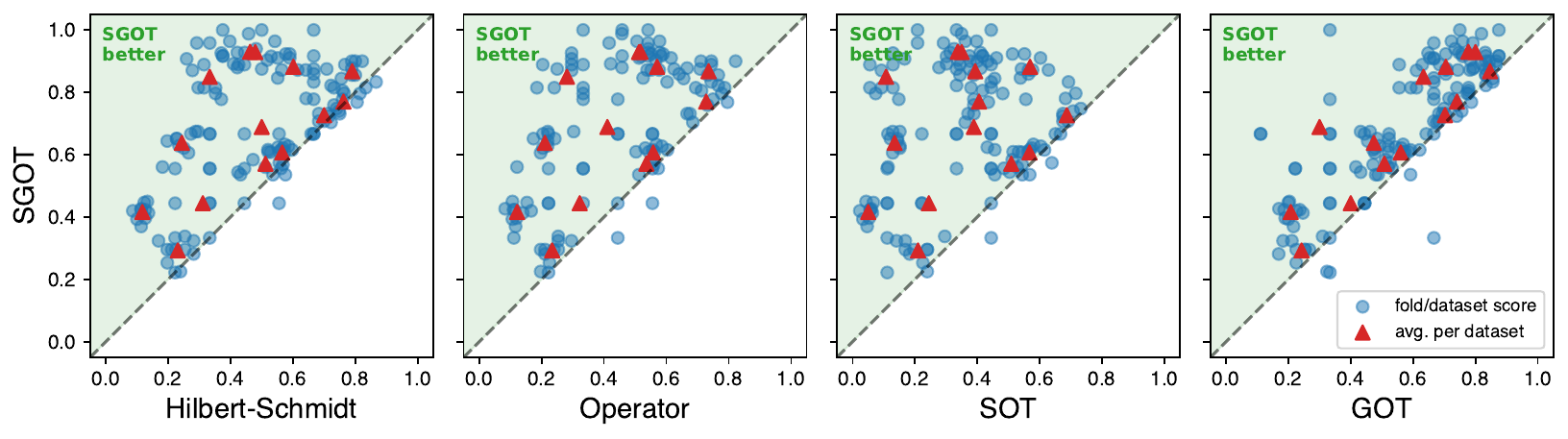}
      \vspace{-3mm}
    \caption{Classification performance (accuracy) comparison between SGOT and competitive metrics. Each point represents a dataset accuracy, with SGOT on the y-axis and the competing metrics on the x-axis.}
    \label{fig: classfication score plot}
\end{figure}

\textbf{Classification experiment.} We now quantify similarities'
performances on a classification task. We perform a 10-iteration Monte-Carlo nested cross-validation for all datasets with a (0.7,0.3)-train/test split ratio and no data preprocessing. We train K-NN classifiers with
each similarity measure and validate the parameters $K$ (and
the $\eta$-parameter for SGOT) with a 5-fold inner cross-validation. We report in \Cref{fig: classfication score plot} the accuracy performances between our metric SGOT and other metrics. Compared to Hilbert-Schmidt on operator metrics, SGOT performances are either equivalent or better by a large margin (large off-diagonal gap). Importantly, SGOT outperforms SOT and GOT by combining information on spectral and eigensubspaces. In \Cref{appendix: ML experiment}, we provide further details on the experimental protocol and additional results, including a full table of performances, critical diagram difference, and execution time.

\subsection{Barycenters and interpolation of dynamical systems}

\begin{figure}[t]
    \centering %
    \vspace{-2mm}
    \includegraphics[width=.9\linewidth]{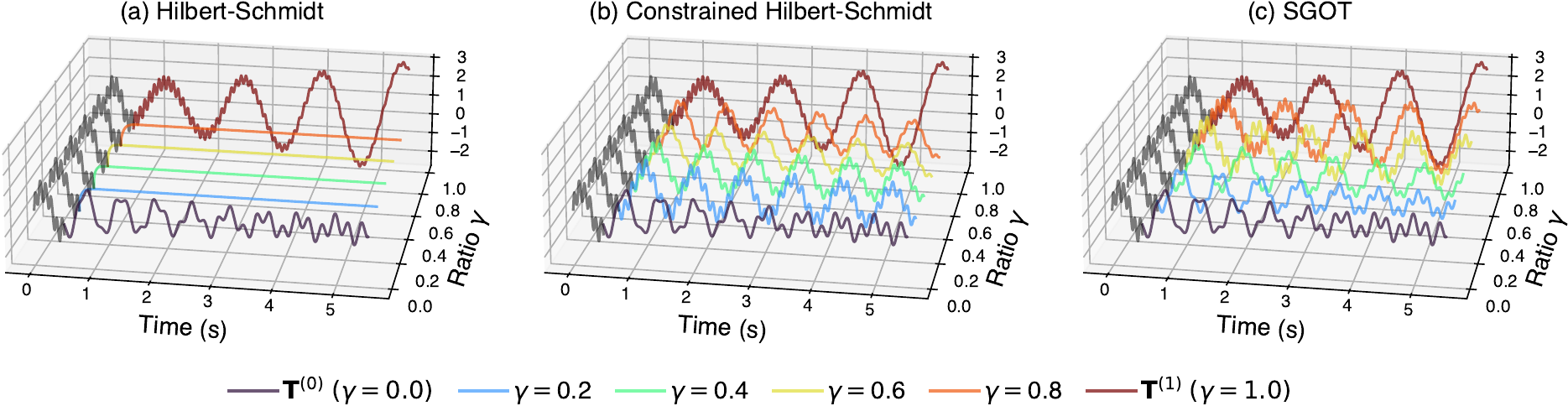}
    \vspace{-3mm}
    \caption{Predictions of interpolated systems between two linear oscillating systems from the same initialization. Interpolated systems correspond to weighted Fréchet barycenter for three different metrics: (a) Hilbert-Schmidt, (b) Hilbert-Schmidt with spectral decomposition constraints, and (c) our metric SGOT. The interpolation is controlled by a ratio parameter $\gamma \in [0,1]$ which sets operators' weights.}
    \label{fig: interpolation}\vspace{-4mm}
\end{figure}

\textbf{Interpolation between 1D DS.} In this experiment, we compare the interpolation between dynamical systems through the weighted Fréchet barycenters of their Koopman operators, estimated with a linear kernel, for different metrics. The two systems are linear oscillatory systems, each being the sum of two simple harmonic oscillators with different frequencies and decay rates, and additive Gaussian noise. The interpolation is controlled by a ratio parameter $\gamma \in [0,1]$ with weights ($1-\gamma, \gamma$) in the Fréchet mean problem \eqref{eq:frechet_mean_problem}. We compare (a) the Hilbert-Schmidt metric without spectral decomposition constraints given by $\mbT_{bar} = (1-\gamma) \mbT^{(0)} + \gamma \mbT^{(1)}$, (b) the Hilbert-Schmidt metric with spectral decomposition constraints, and (c) our proposed metric SGOT. For (b) and (c), barycentric operators are estimated with the proposed optimization scheme, and experimental settings are detailed in \Cref{appendix: barycenter experiment}.
\\
The interpolated predictions, starting from an
identical initialization signal (in gray) containing all four frequencies, are illustrated in \Cref{fig: interpolation} for all
three metrics.  In the Hilbert-Schmidt case (\cref{fig: interpolation}.a) leads to overdamped systems $\forall0<\gamma<1$. Adding spectral decomposition constraints on the Hilbert-Schmidt barycenter (\cref{fig: interpolation}. (b) mitigates the damping effects; however, the oscillatory frequencies and decay rate converge to a local minimum close to initialization, as expected by the saturating behavior of the Hilbert-Schmidt metric (see \cref{fig: ablation study}).  Only SGOT barycenters naturally interpolate between the two systems, notably by retrieving the frequencies and the decay rates.

\begin{figure}[t]
    \centering \vspace{-5mm}
    \includegraphics[width=.9\linewidth]{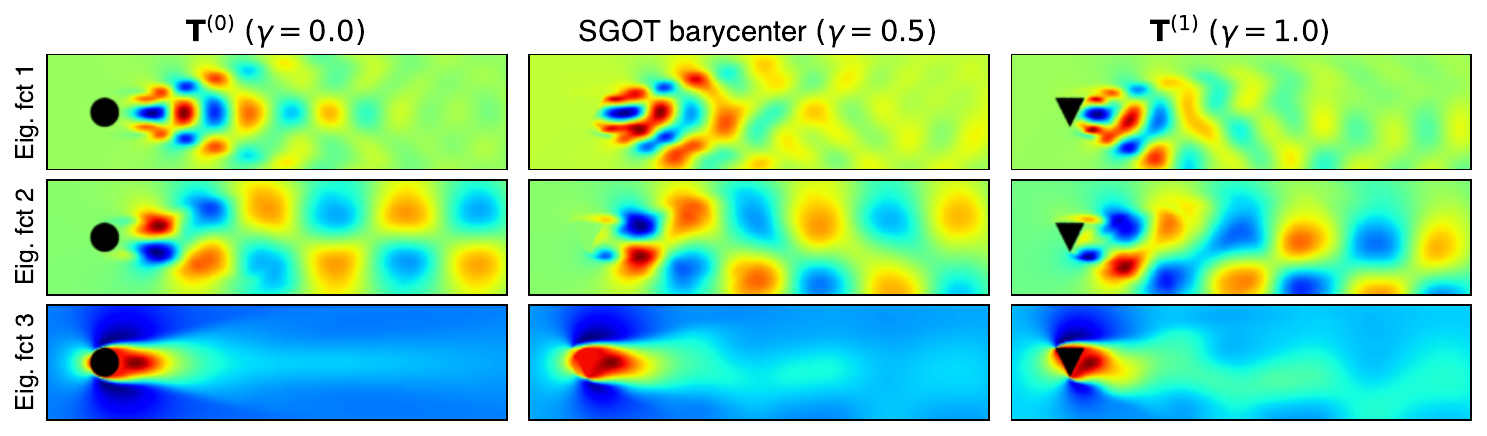}
    \caption{SGOT barycenter of Koopman operators of flows past static objects: a cylinder $\mbT^{(0)}$ and a triangle $\mbT^{(1)}$. Each operator's three leading right eigenfunctions are displayed and can be associated with the vortex-shedding phenomenon of the fluids flowing from left to right.}
    \vspace{-2mm}
    \label{fig:fluid barycenter}
\end{figure}

\textbf{Interpolating fluid dynamics.}
We aim to compute the barycenter of two fluid dynamics systems. To that end, we consider the \textit{Flow past a bluff object} dataset~\citep{tali2025flowbench}, which gathers trajectories of time-varying 2D velocity and pressure fields of incompressible Navier-Stokes fluids flowing around static objects. We select two trajectories, one with a cylinder object and the other with a triangular object. We only kept the velocity field along the flowing direction for each trajectory, leading to trajectories containing 242 samples of 1024x256 grids, which we down-sampled to grids with a 256x64 resolution. We estimate a Koopman operator with linear kernel using the RRR method from each trajectory: $\mbT^{(0)}$ for the cylinder and $\mbT^{(1)}$ for the triangle. The operators are restricted to the fourth leading eigenvalues and eigenfunctions. We compute the SGOT barycenter with the optimization scheme described in \Cref{appendix: barycenter} with an initialization being the average of eigenvalues and eigenfunctions. In \Cref{appendix: barycenter experiment} we detail the experimental settings.  \\
\Cref{fig:fluid barycenter} illustrates the non-conjugated right eigenfunctions of all three Koopman operators (cylinder, barycenter, triangle). By symmetry of boundary conditions and the cylinder, the eigenfunctions in the cylinder case have an axial symmetry that is lost with the triangle. SGOT by interpolating between both introduces the asymmetry in the eigenfunctions of the barycenter.

\section{Conclusion}

In this paper, we proposed SGOT, a novel optimal transport metric between distributional representations of transfer operators in the joint spectral–Grassmann space. The metric has strong theoretical properties, induces a meaningful geometry for barycenters and interpolation, and can be computed efficiently. Numerical experiments demonstrate the superiority of the proposed metric for machine learning tasks and system interpolation. Our method opens the door to machine learning applications on dynamical systems, with future work including dictionary learning and conditional prediction to accelerate numerical simulations.

\section*{Acknowledgments}
This project received funding from the European Union’s Horizon Europe research and innovation program under grant agreement 101120237 (ELIAS), Fondation de l’Ecole Polytechnique, NextGenerationEU and MUR PNRR project PE0000013 CUP J53C22003010006 “Future Artificial Intelligence Research (FAIR)”.

\bibliography{arxiv}

\begin{thebibliography}{74}
\providecommand{\natexlab}[1]{#1}
\providecommand{\url}[1]{\texttt{#1}}
\expandafter\ifx\csname urlstyle\endcsname\relax
  \providecommand{\doi}[1]{doi: #1}\else
  \providecommand{\doi}{doi: \begingroup \urlstyle{rm}\Url}\fi

\bibitem[Afsari \& Vidal(2013)Afsari and Vidal]{afsari2013alignment}
Bijan Afsari and Ren{\'e} Vidal.
\newblock The alignment distance on spaces of linear dynamical systems.
\newblock In \emph{52nd IEEE Conference on Decision and Control}, pp.\  1162--1167. IEEE, 2013.

\bibitem[Afsari \& Vidal(2014)Afsari and Vidal]{afsari2014distances}
Bijan Afsari and Ren{\'e} Vidal.
\newblock Distances on spaces of high-dimensional linear stochastic processes: A survey.
\newblock In \emph{Geometric Theory of Information}, pp.\  219--242. Springer, 2014.

\bibitem[Agueh \& Carlier(2011)Agueh and Carlier]{agueh2011barycenters}
Martial Agueh and Guillaume Carlier.
\newblock Barycenters in the wasserstein space.
\newblock \emph{SIAM Journal on Mathematical Analysis}, 43\penalty0 (2):\penalty0 904--924, 2011.

\bibitem[Alexander \& Giannakis(2020)Alexander and Giannakis]{Alexander2020}
Romeo Alexander and Dimitrios Giannakis.
\newblock Operator-theoretic framework for forecasting nonlinear time series with kernel analog techniques.
\newblock \emph{Physica D: Nonlinear Phenomena}, 409:\penalty0 132520, 2020.

\bibitem[{\'A}lvarez-Esteban et~al.(2016){\'A}lvarez-Esteban, Del~Barrio, Cuesta-Albertos, and Matr{\'a}n]{alvarez2016fixed}
Pedro~C {\'A}lvarez-Esteban, E~Del~Barrio, JA~Cuesta-Albertos, and C~Matr{\'a}n.
\newblock A fixed-point approach to barycenters in wasserstein space.
\newblock \emph{Journal of Mathematical Analysis and Applications}, 441\penalty0 (2):\penalty0 744--762, 2016.

\bibitem[Anderes et~al.(2016)Anderes, Borgwardt, and Miller]{anderes2016discrete}
Ethan Anderes, Steffen Borgwardt, and Jacob Miller.
\newblock Discrete wasserstein barycenters: Optimal transport for discrete data.
\newblock \emph{Mathematical Methods of Operations Research}, 84\penalty0 (2):\penalty0 389--409, 2016.

\bibitem[Andruchow(2014)]{andruchow2014grassmann}
Esteban Andruchow.
\newblock The grassmann manifold of a hilbert space.
\newblock 2014.

\bibitem[Antonini \& Cavalletti(2021)Antonini and Cavalletti]{antonini2021geometry}
Paolo Antonini and Fabio Cavalletti.
\newblock Geometry of grassmannians and optimal transport of quantum states.
\newblock \emph{arXiv preprint arXiv:2104.02616}, 2021.

\bibitem[Bendokat et~al.(2024)Bendokat, Zimmermann, and Absil]{bendokat2024grassmann}
Thomas Bendokat, Ralf Zimmermann, and P-A Absil.
\newblock A grassmann manifold handbook: Basic geometry and computational aspects.
\newblock \emph{Advances in Computational Mathematics}, 50\penalty0 (1):\penalty0 6, 2024.

\bibitem[Bevanda et~al.(2023)Bevanda, Beier, Lederer, Sosnowski, H{\"u}llermeier, and Hirche]{bevanda2023koopman}
Petar Bevanda, Max Beier, Armin Lederer, Stefan Sosnowski, Eyke H{\"u}llermeier, and Sandra Hirche.
\newblock Koopman kernel regression.
\newblock \emph{Advances in Neural Information Processing Systems}, 36:\penalty0 16207--16221, 2023.

\bibitem[Bissacco et~al.(2007)Bissacco, Chiuso, and Soatto]{bissacco2007classification}
Alessandro Bissacco, Alessandro Chiuso, and Stefano Soatto.
\newblock Classification and recognition of dynamical models: The role of phase, independent components, kernels and optimal transport.
\newblock \emph{IEEE Transactions on Pattern Analysis and Machine Intelligence}, 29\penalty0 (11):\penalty0 1958--1972, 2007.

\bibitem[Bonneel et~al.(2011)Bonneel, Van De~Panne, Paris, and Heidrich]{bonneel2011displacement}
Nicolas Bonneel, Michiel Van De~Panne, Sylvain Paris, and Wolfgang Heidrich.
\newblock Displacement interpolation using lagrangian mass transport.
\newblock In \emph{Proceedings of the 2011 SIGGRAPH Asia conference}, pp.\  1--12, 2011.

\bibitem[Bruder et~al.(2020)Bruder, Fu, Gillespie, Remy, and Vasudevan]{bruder2020data}
Daniel Bruder, Xun Fu, R~Brent Gillespie, C~David Remy, and Ram Vasudevan.
\newblock Data-driven control of soft robots using koopman operator theory.
\newblock \emph{IEEE transactions on robotics}, 37\penalty0 (3):\penalty0 948--961, 2020.

\bibitem[Brunton et~al.(2022)Brunton, Budi{\v{s}}i{\'{c}}, Kaiser, and Kutz]{brunton2021modern}
Steven~L. Brunton, Marko Budi{\v{s}}i{\'{c}}, Eurika Kaiser, and J.~Nathan Kutz.
\newblock Modern {K}oopman theory for dynamical systems.
\newblock \emph{{SIAM} Review}, 64\penalty0 (2):\penalty0 229--340, 2022.

\bibitem[Caponnetto \& De~Vito(2007)Caponnetto and De~Vito]{caponnetto2007}
Andrea Caponnetto and Ernesto De~Vito.
\newblock Optimal rates for the regularized least-squares algorithm.
\newblock \emph{Foundations of Computational Mathematics}, 7\penalty0 (3):\penalty0 331--368, 2007.

\bibitem[Chan \& Vasconcelos(2005)Chan and Vasconcelos]{chan2005probabilistic}
Antoni~B Chan and Nuno Vasconcelos.
\newblock Probabilistic kernels for the classification of auto-regressive visual processes.
\newblock In \emph{2005 IEEE Computer Society Conference on Computer Vision and Pattern Recognition (CVPR'05)}, volume~1, pp.\  846--851. IEEE, 2005.

\bibitem[Chaudhry \& Vidal(2013)Chaudhry and Vidal]{chaudhry2013initial}
Rizwan Chaudhry and Ren{\'e} Vidal.
\newblock Initial-state invariant binet-cauchy kernels for the comparison of linear dynamical systems.
\newblock In \emph{52nd IEEE Conference on Decision and Control}, pp.\  5377--5384. IEEE, 2013.

\bibitem[Claici et~al.(2018)Claici, Chien, and Solomon]{claici2018stochastic}
Sebastian Claici, Edward Chien, and Justin Solomon.
\newblock Stochastic wasserstein barycenters.
\newblock In \emph{International Conference on Machine Learning}, pp.\  999--1008. PMLR, 2018.

\bibitem[Colbrook et~al.(2023)Colbrook, Ayton, and Sz{\H{o}}ke]{colbrook2023residual}
Matthew~J Colbrook, Lorna~J Ayton, and M{\'a}t{\'e} Sz{\H{o}}ke.
\newblock Residual dynamic mode decomposition: robust and verified koopmanism.
\newblock \emph{Journal of Fluid Mechanics}, 955:\penalty0 A21, 2023.

\bibitem[Colbrook et~al.(2025)Colbrook, Drysdale, and Horning]{colbrook2025rigged}
Matthew~J Colbrook, Catherine Drysdale, and Andrew Horning.
\newblock Rigged dynamic mode decomposition: Data-driven generalized eigenfunction decompositions for koopman operators.
\newblock \emph{SIAM Journal on Applied Dynamical Systems}, 24\penalty0 (2):\penalty0 1150--1190, 2025.

\bibitem[Cuturi \& Doucet(2014)Cuturi and Doucet]{cuturi2014fast}
Marco Cuturi and Arnaud Doucet.
\newblock Fast computation of wasserstein barycenters.
\newblock In \emph{International conference on machine learning}, pp.\  685--693. PMLR, 2014.

\bibitem[De~Cock \& De~Moor(2002)De~Cock and De~Moor]{de2002subspace}
Katrien De~Cock and Bart De~Moor.
\newblock Subspace angles between arma models.
\newblock \emph{Systems \& Control Letters}, 46\penalty0 (4):\penalty0 265--270, 2002.

\bibitem[Dunford \& Schwartz(1988)Dunford and Schwartz]{dunford1988linear}
Nelson Dunford and Jacob~T Schwartz.
\newblock \emph{Linear operators, part 1: general theory}.
\newblock John Wiley \& Sons, 1988.

\bibitem[Fujii et~al.(2017)Fujii, Inaba, and Kawahara]{fujii2017koopman}
Keisuke Fujii, Yuki Inaba, and Yoshinobu Kawahara.
\newblock Koopman spectral kernels for comparing complex dynamics: Application to multiagent sport plays.
\newblock In \emph{Joint European Conference on Machine Learning and Knowledge Discovery in Databases}, pp.\  127--139. Springer, 2017.

\bibitem[Georgiou(2007)]{georgiou2007distances}
Tryphon~T Georgiou.
\newblock Distances and riemannian metrics for spectral density functions.
\newblock \emph{IEEE Transactions on Signal Processing}, 55\penalty0 (8):\penalty0 3995--4003, 2007.

\bibitem[Glaz(2025)]{glaz2025efficient}
Bryan Glaz.
\newblock Efficient pseudometrics for data-driven comparisons of nonlinear dynamical systems.
\newblock \emph{Nonlinear Dynamics}, 113\penalty0 (11):\penalty0 12465--12486, 2025.

\bibitem[Gray(2009)]{gray2009probability}
Robert~M Gray.
\newblock \emph{Probability, random processes, and ergodic properties}.
\newblock Springer Science \& Business Media, 2009.

\bibitem[Han et~al.(2024)Han, Liu, and Zhu]{han2024geometry}
Bang-Xian Han, Deng-Yu Liu, and Zhuo-Nan Zhu.
\newblock On the geometry of wasserstein barycenter i.
\newblock \emph{arXiv preprint arXiv:2412.01190}, 2024.

\bibitem[Hanzon \& Marcus(1982)Hanzon and Marcus]{hanzon1982riemannian}
Bernard Hanzon and Steven~I Marcus.
\newblock Riemannian metrics on spaces of stable linear systems, with applications to identification.
\newblock In \emph{1982 21st IEEE Conference on Decision and Control}, pp.\  1119--1124. IEEE, 1982.

\bibitem[Ishikawa et~al.(2018)Ishikawa, Fujii, Ikeda, Hashimoto, and Kawahara]{ishikawa2018metric}
Isao Ishikawa, Keisuke Fujii, Masahiro Ikeda, Yuka Hashimoto, and Yoshinobu Kawahara.
\newblock Metric on nonlinear dynamical systems with perron-frobenius operators.
\newblock \emph{Advances in Neural Information Processing Systems}, 31, 2018.

\bibitem[Kato(2013)]{kato2013perturbation}
Tosio Kato.
\newblock \emph{Perturbation theory for linear operators}, volume 132.
\newblock Springer Science \& Business Media, 2013.

\bibitem[Kingma \& Ba(2014)Kingma and Ba]{Kingma2014AdamAM}
Diederik~P. Kingma and Jimmy Ba.
\newblock Adam: A method for stochastic optimization.
\newblock \emph{CoRR}, abs/1412.6980, 2014.
\newblock URL \url{https://api.semanticscholar.org/CorpusID:6628106}.

\bibitem[Klabunde et~al.(2025)Klabunde, Schumacher, Strohmaier, and Lemmerich]{klabunde2025similarity}
Max Klabunde, Tobias Schumacher, Markus Strohmaier, and Florian Lemmerich.
\newblock Similarity of neural network models: A survey of functional and representational measures.
\newblock \emph{ACM Computing Surveys}, 57\penalty0 (9):\penalty0 1--52, 2025.

\bibitem[Klus \& Djurdjevac~Conrad(2023)Klus and Djurdjevac~Conrad]{klus2023koopman}
Stefan Klus and Nata{\v{s}}a Djurdjevac~Conrad.
\newblock Koopman-based spectral clustering of directed and time-evolving graphs.
\newblock \emph{Journal of nonlinear science}, 33\penalty0 (1):\penalty0 8, 2023.

\bibitem[Kostic et~al.(2024{\natexlab{a}})Kostic, Lounici, Halconruy, Devergne, and Pontil]{Kostic-Sub2024}
V.~R. Kostic, K.~Lounici, H.~Halconruy, T.~Devergne, and M.~Pontil.
\newblock Learning the infinitesimal generator of stochastic diffusion processes.
\newblock In \emph{Conference on Neural Information Processing Systems}, volume~38, 2024{\natexlab{a}}.

\bibitem[Kostic et~al.(2024{\natexlab{b}})Kostic, Novelli, Grazzi, Lounici, and Pontil]{Kostic-ICLR2024}
V.~R. Kostic, P.~Novelli, R.~Grazzi, K.~Lounici, and M.~Pontil.
\newblock Learning invariant representations of time-homogeneous stochastic dynamical systems.
\newblock In \emph{International Conference on Learning Representations}, 2024{\natexlab{b}}.

\bibitem[Kostic et~al.(2022)Kostic, Novelli, Maurer, Ciliberto, Rosasco, and Pontil]{kostic2022learning}
Vladimir Kostic, Pietro Novelli, Andreas Maurer, Carlo Ciliberto, Lorenzo Rosasco, and Massimiliano Pontil.
\newblock Learning dynamical systems via koopman operator regression in reproducing kernel hilbert spaces.
\newblock \emph{Advances in Neural Information Processing Systems}, 35:\penalty0 4017--4031, 2022.

\bibitem[Kostic et~al.(2023)Kostic, Lounici, Novelli, and Pontil]{kostic2023sharp}
Vladimir Kostic, Karim Lounici, Pietro Novelli, and Massimiliano Pontil.
\newblock Sharp spectral rates for koopman operator learning.
\newblock \emph{Advances in Neural Information Processing Systems}, 36:\penalty0 32328--32339, 2023.

\bibitem[Kostic et~al.(2024{\natexlab{c}})Kostic, Inzerili, Lounici, Novelli, and Pontil]{kostic2024consistent}
Vladimir Kostic, Prune Inzerili, Karim Lounici, Pietro Novelli, and Massimiliano Pontil.
\newblock Consistent long-term forecasting of ergodic dynamical systems.
\newblock In \emph{2024 International Conference on Machine Learning}, 2024{\natexlab{c}}.

\bibitem[Kutz et~al.(2016)Kutz, Brunton, Brunton, and Proctor]{kutz2016dynamic}
J~Nathan Kutz, Steven~L Brunton, Bingni~W Brunton, and Joshua~L Proctor.
\newblock \emph{Dynamic mode decomposition: data-driven modeling of complex systems}.
\newblock SIAM, 2016.

\bibitem[Lange et~al.(2021)Lange, Brunton, and Kutz]{lange2021fourier}
Henning Lange, Steven~L Brunton, and J~Nathan Kutz.
\newblock From fourier to koopman: Spectral methods for long-term time series prediction.
\newblock \emph{Journal of Machine Learning Research}, 22\penalty0 (41):\penalty0 1--38, 2021.

\bibitem[Lasota \& Mackey(1994)Lasota and Mackey]{Lasota1994}
Andrzej Lasota and Michael~C. Mackey.
\newblock \emph{Chaos, Fractals, and Noise}, volume~97 of \emph{Applied Mathematical Sciences}.
\newblock Springer New York, 1994.

\bibitem[Lasota \& Mackey(2013)Lasota and Mackey]{lasota2013chaos}
Andrzej Lasota and Michael~C Mackey.
\newblock \emph{Chaos, fractals, and noise: stochastic aspects of dynamics}, volume~97.
\newblock Springer Science \& Business Media, 2013.

\bibitem[Le~Gouic \& Loubes(2017)Le~Gouic and Loubes]{le2017existence}
Thibaut Le~Gouic and Jean-Michel Loubes.
\newblock Existence and consistency of wasserstein barycenters.
\newblock \emph{Probability Theory and Related Fields}, 168\penalty0 (3):\penalty0 901--917, 2017.

\bibitem[Li et~al.(2020)Li, Genevay, Yurochkin, and Solomon]{li2020continuous}
Lingxiao Li, Aude Genevay, Mikhail Yurochkin, and Justin~M Solomon.
\newblock Continuous regularized wasserstein barycenters.
\newblock \emph{Advances in Neural Information Processing Systems}, 33:\penalty0 17755--17765, 2020.

\bibitem[Lindheim(2023)]{lindheim2023simple}
Johannes~von Lindheim.
\newblock Simple approximative algorithms for free-support wasserstein barycenters.
\newblock \emph{Computational Optimization and Applications}, 85\penalty0 (1):\penalty0 213--246, 2023.

\bibitem[Liu et~al.(2024)Liu, Sholokhov, Mansour, and Nabi]{Liu2024}
Yuying Liu, Aleksei Sholokhov, Hassan Mansour, and Saleh Nabi.
\newblock Physics-informed koopman network for time-series prediction of dynamical systems.
\newblock In \emph{ICLR 2024 Workshop on AI4DifferentialEquations In Science}, 2024.

\bibitem[Maaten \& Hinton(2008)Maaten and Hinton]{maaten2008visualizing}
Laurens van~der Maaten and Geoffrey Hinton.
\newblock Visualizing data using t-sne.
\newblock \emph{Journal of machine learning research}, 9\penalty0 (Nov):\penalty0 2579--2605, 2008.

\bibitem[Mallasto \& Feragen(2017)Mallasto and Feragen]{mallasto2017learning}
Anton Mallasto and Aasa Feragen.
\newblock Learning from uncertain curves: The 2-wasserstein metric for gaussian processes.
\newblock \emph{Advances in Neural Information Processing Systems}, 30, 2017.

\bibitem[Martin(2002)]{martin2002metric}
Richard~J Martin.
\newblock A metric for arma processes.
\newblock \emph{IEEE transactions on Signal Processing}, 48\penalty0 (4):\penalty0 1164--1170, 2002.

\bibitem[Masarotto et~al.(2019)Masarotto, Panaretos, and Zemel]{masarotto2019procrustes}
Valentina Masarotto, Victor~M Panaretos, and Yoav Zemel.
\newblock Procrustes metrics on covariance operators and optimal transportation of gaussian processes.
\newblock \emph{Sankhya A}, 81\penalty0 (1):\penalty0 172--213, 2019.

\bibitem[Mauroy et~al.(2020)Mauroy, Susuki, and Mezic]{mauroy2020koopman}
Alexandre Mauroy, Y~Susuki, and Igor Mezic.
\newblock \emph{Koopman operator in systems and control}, volume~7.
\newblock Springer, 2020.

\bibitem[Meanti et~al.(2023)Meanti, Chatalic, Kostic, Novelli, Pontil, and Rosasco]{Meanti2023}
G.~Meanti, A.~Chatalic, V.~R. Kostic, P.~Novelli, M.~Pontil, and L.~Rosasco.
\newblock Estimating koopman operators with sketching to provably learn large scale dynamical systems.
\newblock In \emph{Conference on Neural Information Processing Systems}, 2023.

\bibitem[Meyn \& Tweedie(2012)Meyn and Tweedie]{meyn2012markov}
Sean~P Meyn and Richard~L Tweedie.
\newblock \emph{Markov chains and stochastic stability}.
\newblock Springer Science \& Business Media, 2012.

\bibitem[Mezic(2016)]{mezic2016comparison}
Igor Mezic.
\newblock On comparison of dynamics of dissipative and finite-time systems using koopman operator methods.
\newblock \emph{IFAC-PapersOnLine}, 49\penalty0 (18):\penalty0 454--461, 2016.

\bibitem[Mezi{\'c} \& Banaszuk(2004)Mezi{\'c} and Banaszuk]{mezic2004comparison}
Igor Mezi{\'c} and Andrzej Banaszuk.
\newblock Comparison of systems with complex behavior.
\newblock \emph{Physica D: Nonlinear Phenomena}, 197\penalty0 (1-2):\penalty0 101--133, 2004.

\bibitem[Ostrow et~al.(2023)Ostrow, Eisen, Kozachkov, and Fiete]{ostrow2023beyond}
Mitchell Ostrow, Adam Eisen, Leo Kozachkov, and Ila Fiete.
\newblock Beyond geometry: Comparing the temporal structure of computation in neural circuits with dynamical similarity analysis.
\newblock \emph{Advances in Neural Information Processing Systems}, 36:\penalty0 33824--33837, 2023.

\bibitem[Peyr{\'e} et~al.(2019)Peyr{\'e}, Cuturi, et~al.]{peyre2019computational}
Gabriel Peyr{\'e}, Marco Cuturi, et~al.
\newblock Computational optimal transport: With applications to data science.
\newblock \emph{Foundations and Trends{\textregistered} in Machine Learning}, 11\penalty0 (5-6):\penalty0 355--607, 2019.

\bibitem[Qian \& Pan(2021)Qian and Pan]{qian2021inexact}
Yitian Qian and Shaohua Pan.
\newblock An inexact pam method for computing wasserstein barycenter with unknown supports.
\newblock \emph{Computational and Applied Mathematics}, 40\penalty0 (2):\penalty0 45, 2021.

\bibitem[Redman et~al.(2024)Redman, Bello-Rivas, Fonoberova, Mohr, Kevrekidis, and Mezic]{redman2024identifying}
William Redman, Juan Bello-Rivas, Maria Fonoberova, Ryan Mohr, Yannis Kevrekidis, and Igor Mezic.
\newblock Identifying equivalent training dynamics.
\newblock \emph{Advances in Neural Information Processing Systems}, 37:\penalty0 23603--23629, 2024.

\bibitem[Redman et~al.(2022)Redman, Fonoberova, Mohr, Kevrekidis, and Mezi{\'c}]{redman2022algorithmic}
William~T Redman, Maria Fonoberova, Ryan Mohr, Ioannis~G Kevrekidis, and Igor Mezi{\'c}.
\newblock Algorithmic (semi-) conjugacy via koopman operator theory.
\newblock In \emph{2022 IEEE 61st Conference on Decision and Control (CDC)}, pp.\  6006--6011. IEEE, 2022.

\bibitem[Ross(1995)]{ross1995stochastic}
Sheldon~M Ross.
\newblock \emph{Stochastic Processes}.
\newblock John Wiley \& Sons, 1995.

\bibitem[Ruiz et~al.(2021)Ruiz, Flynn, Large, Middlehurst, and Bagnall]{ruiz2021great}
Alejandro~Pasos Ruiz, Michael Flynn, James Large, Matthew Middlehurst, and Anthony Bagnall.
\newblock The great multivariate time series classification bake off: a review and experimental evaluation of recent algorithmic advances.
\newblock \emph{Data mining and knowledge discovery}, 35\penalty0 (2):\penalty0 401--449, 2021.

\bibitem[Sakata \& Kawahara(2024)Sakata and Kawahara]{sakata2024enhancing}
Itsushi Sakata and Yoshinobu Kawahara.
\newblock Enhancing spectral analysis in nonlinear dynamics with pseudoeigenfunctions from continuous spectra.
\newblock \emph{Scientific Reports}, 14\penalty0 (1):\penalty0 19276, 2024.

\bibitem[Sinha et~al.(2024)Sinha, Nandanoori, Huang, Ramachandran, and Bakker]{sinha2024formalisation}
Subhrajit Sinha, Sai~Pushpak Nandanoori, Bowen Huang, Thiagarajan Ramachandran, and Craig Bakker.
\newblock On formalisation of martin distance for linear dynamical systems.
\newblock In \emph{2024 American Control Conference (ACC)}, pp.\  1243--1248. IEEE, 2024.

\bibitem[Steinwart \& Christmann(2008)Steinwart and Christmann]{Steinwart2008}
Ingo Steinwart and Andreas Christmann.
\newblock \emph{Support Vector Machines}.
\newblock Springer New York, 2008.

\bibitem[Surana(2020)]{surana2020koopman}
Amit Surana.
\newblock Koopman operator framework for time series modeling and analysis.
\newblock \emph{Journal of Nonlinear Science}, 30\penalty0 (5):\penalty0 1973--2006, 2020.

\bibitem[Tali et~al.(2025)Tali, Rabeh, Yang, Shadkhah, Karki, Upadhyaya, Dhakshinamoorthy, Saadati, Sarkar, Krishnamurthy, et~al.]{tali2025flowbench}
Ronak Tali, Ali Rabeh, Cheng-Hau Yang, Mehdi Shadkhah, Samundra Karki, Abhisek Upadhyaya, Suriya Dhakshinamoorthy, Marjan Saadati, Soumik Sarkar, Adarsh Krishnamurthy, et~al.
\newblock Flowbench: A large scale benchmark for flow simulation over complex geometries.
\newblock \emph{Journal of Data-centric Machine Learning Research}, 2025.

\bibitem[Villani et~al.(2008)]{villani2008optimal}
C{\'e}dric Villani et~al.
\newblock \emph{Optimal transport: old and new}, volume 338.
\newblock Springer, 2008.

\bibitem[Vishwanathan et~al.(2007)Vishwanathan, Smola, and Vidal]{vishwanathan2007binet}
SVN Vishwanathan, Alexander~J Smola, and Ren{\'e} Vidal.
\newblock Binet-cauchy kernels on dynamical systems and its application to the analysis of dynamic scenes.
\newblock \emph{International Journal of Computer Vision}, 73\penalty0 (1):\penalty0 95--119, 2007.

\bibitem[Wright(2015)]{wright2015coordinate}
Stephen~J Wright.
\newblock Coordinate descent algorithms.
\newblock \emph{Mathematical programming}, 151\penalty0 (1):\penalty0 3--34, 2015.

\bibitem[Wu et~al.(2017)Wu, N{\"u}ske, Paul, Klus, Koltai, and No{\'e}]{wu2017variational}
Hao Wu, Feliks N{\"u}ske, Fabian Paul, Stefan Klus, P{\'e}ter Koltai, and Frank No{\'e}.
\newblock Variational koopman models: Slow collective variables and molecular kinetics from short off-equilibrium simulations.
\newblock \emph{The Journal of chemical physics}, 146\penalty0 (15), 2017.

\bibitem[Zayed(2018)]{zayed2018advances}
AhmedI Zayed.
\newblock \emph{Advances in Shannon's sampling theory}.
\newblock Routledge, 2018.

\bibitem[Zhu et~al.(2024)Zhu, Guha, Do, Xu, Nguyen, and Zhao]{zhu2024functional}
Jiacheng Zhu, Aritra Guha, Dat Do, Mengdi Xu, XuanLong Nguyen, and Ding Zhao.
\newblock Functional optimal transport: regularized map estimation and domain adaptation for functional data.
\newblock \emph{Journal of Machine Learning Research}, 25\penalty0 (276):\penalty0 1--49, 2024.

\end{thebibliography}
\bibliographystyle{arxiv}

\appendix

\section{Related work}
\label{appendix: related_work}

\paragraph{Metric for linear dynamical systems.}
A substantial body of research addresses the comparison of (stochastic) linear dynamical systems (LDSs) and linear state-space models \citep{afsari2014distances}. Early methods exploit the Riemannian manifold structure of LDS spaces to define meaningful metrics \citep{hanzon1982riemannian}, with related developments in power spectral density spaces \citep{georgiou2007distances}, including approaches based on Wasserstein metrics \citep{gray2009probability}. However, these methods suffer from high computational cost. The Martin distance \citep{martin2002metric} offers a practical alternative, comparing ARMA models via their cepstrum. It has been generalized to state-space models and shown equivalent to metrics based on angles between observability subspaces \citep{de2002subspace,sinha2024formalisation}. Other approaches include kernel-based metrics derived from the Binet-Cauchy theorem \citep{vishwanathan2007binet}, Kullback–Leibler divergence \citep{chan2005probabilistic}, and moment matching \citep{bissacco2007classification}. However, compared to the Martin distance, these metrics are sensitive to
trajectory initial conditions, so extensions have been proposed to address this issue \citep{chaudhry2013initial}.

\paragraph{Extension to nonlinear dynamical systems.} For nonlinear dynamical systems, most work leverages the Koopman framework to linearize dynamics. The Binet-Cauchy kernel has been extended to nonlinear systems \citep{fujii2017koopman} within this context. Another kernel leverages Koopman representation to compare observability subspaces \citep{ishikawa2018metric}. The latest has been used alongside a deep learning method for estimating Koopman operators (ResDMD \citet{colbrook2023residual}) in the case of continuous spectrum \citep{sakata2024enhancing}. However, both kernels are sensitive to trajectory initial conditions like the linear case. In \citet{mezic2004comparison}, the authors propose metrics to compare the asymptotic dynamics of measure-preserving systems via Koopman representations, later extended to dissipative systems over finite time \citep{mezic2016comparison}.

\paragraph{Metric for topologically conjugated dynamical systems.}
Recently, interest has grown in comparing neural network dynamics in neuroscience and deep learning \citep{klabunde2025similarity}. Such comparisons often consider topological conjugacy, leading to metrics on quotient spaces. \citet{redman2022algorithmic,redman2024identifying} show that topologically conjugate systems share identical Koopman spectra and propose a pseudo-metric based on optimal transport. \citet{ostrow2023beyond} extends Procrustes analysis to compare Koopman representations up to orthogonal transformations, extending earlier work in the LDS setting \citep{afsari2013alignment}. \citet{glaz2025efficient} further generalizes these metrics to accommodate broader transformation classes.

\paragraph{Optimal transport on functional spaces.}
A related direction studies measures on functional spaces. Some works have studied measures on Gaussian processes
\citep{masarotto2019procrustes,mallasto2017learning}, for which there exists a closed-form of the metric. In \citet{antonini2021geometry}, the authors propose a theoretical Wasserstein metric between measures derived from the spectral decomposition of normal operators. More recently, \citet{zhu2024functional} introduced a computable approximation of the Wasserstein metric between measures on infinite-dimensional Hilbert spaces, obtained by restriction to linear mappings.

\section{Learning Koopman Transfer Operators with Kernel Methods}
\label{appendix: background_RKHS_Learning_Koopman}

In many practical scenarios, $\Koop{}$ is unknown, but data from system trajectories are available. For such cases, Koopman operator regression in reproducing kernel Hilbert spaces (RKHS) provides a learning framework to estimate $\Koop{}$ on $\Lii$~\cite{kostic2022learning}. Let $\RKHS$ be a RKHS with a bounded kernel $k$ and feature map $\phi$ such that $k(x,y) = \scalarp{\phi(x),\phi(y)}$. We recall that the injection operator $\TS : \RKHS \to \Lii$ is Hilbert-Schmidt \cite{caponnetto2007,Steinwart2008}, and thus so is the restricted Koopman operator $\TZ := \Koop{}\TS : \RKHS\to\Lii$.

The goal is to approximate $\TZ = \Koop{}\TS$ by minimizing the risk $\Risk(\Estim)= \EE_{x\sim \im} \sum_{i \in \N} \EE \big[ (h_i(X_{t+1}) - (\Estim h_i)(X_t))^2\,\vert X_t = x\big]$ over Hilbert-Schmidt operators $\Estim\in\HS{\RKHS}$, where $(h_i)_{i\in\N}$ is an orthonormal basis of $\RKHS$. This risk admits a decomposition $\Risk(\Estim)=\IrRisk + \ExRisk(\Estim)$, where
\begin{equation}
    \label{eq:ex_ir_risk}
\IrRisk=\hnorm{\TS}^2-\hnorm{\TZ}^2\geq0\quad \text{and} \quad \ExRisk(\Estim)= \hnorm{\Koop{} \TS-\TS\Estim}^2= \lVert\Koop{} \TS-\TS\Estim\rVert_{\rm{HS}(\RKHS,\Lii
)}^2
\end{equation}
are the irreducible risk and the excess risk, respectively. Using universal kernels, the excess risk can be made arbitrarily small: $\inf_{\Estim\in\HS{\RKHS}} \ExRisk(\Estim)= 0$.

A common approach is to solve the Tikhonov-regularized problem
\begin{equation}\label{eq:KOR_reg}
\min_{\Estim \in \HS{\RKHS}} \Risk^\reg(\Estim){:=}\Risk(\Estim) + \reg\hnorm{\Estim}^2,
\end{equation}
with $\reg>0$. Defining the covariance operator $\Cx := \TS^*\TS = \EE_{x\sim\im} \phi(x)\otimes \phi(x)$ and the cross-covariance operator $\Cxy : = \TS^*\TZ = \EE_{(x,y)\sim\rho} \phi(x)\otimes \phi(y)$ (where $\rho$ is the joint measure of consecutive states), the unique solution to \eqref{eq:KOR_reg} is the Kernel Ridge Regression (KRR) estimator $\RKoop:=\Creg^{-1} \Cxy$, where $\Creg:=\Cx+\reg\Id_\RKHS$.

To approximate the leading eigenvalues of $\Koop{}$, low-rank estimators are used. The Reduced Rank Regression (RRR) estimator~\cite{kostic2022learning} is the solution to \eqref{eq:KOR_reg} under a rank-$r$ constraint:
\begin{equation}\label{eq:KOR_RRR}
\Creg^{-1/2} \SVDr{ \Creg^{-1/2} \Cxy } = \argmin_{\Estim \in \HSr} \Risk^\reg(\Estim),
\end{equation}
where $\HSr$ denotes the set of rank-$r$ HS operators and $\SVDr{\cdot}$ is the $r$-truncated SVD.

Given data $\Data = \{(x_i,y_i)\}_{i\in[n]}$, empirical estimators are derived by minimizing the regularized empirical risk $\ERisk^\reg(\Estim){:=}\frac{1}{n}\sum_{i\in[n]} \norm{\phi(y_i) - \Estim^*\phi(x_i)}^2 + \reg\hnorm{\Estim}^2$. 
Introducing  the sampling operators for data $\Data$ and RKHS $\RKHS$ by
\begin{align*}
    \ES \colon \RKHS \to \R^{n} \quad \text{ s.t. }  f \mapsto \tfrac{1}{\sqrt{n}}[ f(x_{i})]_{i \in[n]} & \quad \text{ and } &  \EZ \colon \RKHS \to \R^{n} \quad \text{ s.t. }  f \mapsto \tfrac{1}{\sqrt{n}}[ f(y_{i})]_{i \in[n]},
\end{align*}
 and their adjoints by
\begin{align*}
   \ES^* \colon \R^{n} \to \RKHS \quad \text{ s.t. } w \mapsto \tfrac{1}{\sqrt{n}}\sum_{i\in[n]}w_i\fH(x_i) & \quad \text{ and } & \EZ^* \colon \R^{n} \to \RKHS \quad \text{ s.t. } w \mapsto \tfrac{1}{\sqrt{n}}\sum_{i\in[n]}w_i\fG(y_i),
\end{align*}
we obtain  $\ERisk^\reg(\Estim){=}\hnorm{\EZ {-} \ES \Estim}^2 + \reg\hnorm{\Estim}^2$.

The empirical covariance and cross-covariance operators are:
\begin{equation}\label{eq:empirical_cov}
\ECx := \ES ^*\ES,\quad \ECy := \EZ ^*\EZ, \quad \ECxy := \ES ^*\EZ.
\end{equation}
The corresponding regularized empirical covariance is $\ECreg : = \ECx+\reg \Id_{\RKHS}$. The kernel Gram matrices are:
\begin{equation}\label{eq:kernel_mx}
\Kx := \ES \ES^*, \quad \Ky := \EZ \EZ^*.
\end{equation}
The empirical RRR estimator is then $\ECreg^{-1/2} \SVDr{ \ECreg^{-1/2} \ECxy }$. These empirical estimators can be expressed in the form $\EEstim = \ES U_r V_r^\top \EZ$ for matrices $U_r,V_r \in\R^{n\times r}$~\cite{kostic2022learning}, enabling the computation of spectral decompositions in infinite-dimensional RKHS.

\begin{theorem}[\cite{kostic2022learning}]\label{thm:spec_decomp_e}
Let $1\leq r \leq n$ and $ \EEstim = \ES U_r V_r^\top \EZ$,  where $U_r,V_r \in\R^{n\times r}$. If $V_r^\top \Kyx U_r \in\R^{r\times r}$, for $\Kyx = n^{-1}[k(y_i,x_j)_{i,j\in[n]}]$, is full rank and non-defective, the spectral decomposition $(\eeval_i,\elefun_i, \erefun_i)_{i\in[r]}$ of $\EEstim$ can be expressed in terms of the spectral decomposition $(\eeval_i,\levec_i, \revec_i)_{i\in[r]}$~of~ $V_r^\top \Kyx U_r $ as $\elefun_i = \eeval_i \EZ^*V_r \levec_i / \abs{\eeval_i}$ and $\erefun_i = \ES^*U_r \revec_i$, for all $i\in[r]$.
\end{theorem}

\paragraph{RKHS embedddings into $\Lii$.}

We recall some facts on the injection operator $\TS$. Note first that $\TS\in\HS{\RKHS,\Lii}$. Then according to the spectral theorem for positive self-adjoint operators, $\TS$ has an SVD, i.e. there exists at most countable positive sequence $(\sigma_j)_{j\in J}$, where $J:=\{1,2,\ldots,\}\subseteq\N$, and ortho-normal systems $(\ell_j)_{j\in J}$ and $(h_j)_{j\in J}$ of $\cl(\range(\TS))$ and $\Ker(\TS)^\perp$, respectively, such that $\TS h_j = \sigma_j \ell_j$ and $\TS^* \ell_j = \sigma_j h_j$, $j\in J$. 

Now, given $\rpar{}\geq 0$, let us define scaled injection operator $S_{\rpar{}} \colon \RKHS{} \to \Lii{}$ as
\begin{equation}\label{eq:injection_scaled}
S_\rpar:= \sum_{j\in J}\sigma_j^{\rpar}\ell_j\otimes h_j.
\end{equation}
Clearly, we have that $\TS = S_1$, while $\range{S_0} = \cl(\range(\TS))$. Next, we equip $\range(S_{\rpar})$ with a norm $\|\cdot\|_\rpar$ to build an interpolation space:
\[
[\RKHS]_\rpar:=\left\{ f\in\range(S_{\rpar})\;\vert\; \|f\|_\rpar^2:= \sum_{j\in J}\sigma_j^{-2 \rpar} \scalarp{f,\ell_j}^2 <\infty \right\}.
\]

\section{Spectral-Grassmann Wasserstein metric (SGOT) proof}
\label{appendix: proof sgot}

\subsection{Main proof}
In this section we prove that $\mcS_r(\mcH)$ can be endowed with a Wasserstein metric based on operator spectral decomposition as summarized by the following theorem:
\begin{thm}\label{thm:main_wasserstein_metric appendix}
Let $\mcH$ be a separable $\mdC$-Hilbert space and $\mcS_r(\mcH)$ the set of non-defective operators with rank at most $r\in\mcD$. Let $(\mcG, d_\mcG)$ be Grassmanian manifold of the space of Hilbert-Schmidt operators on $\RKHS$. Given $p{\in}\mdN^*$ and $\eta {\in} (0,1)$, let $\mu{\colon} S_r(\mcH){\to}\mcP_p(\mdC \times \mcG)$ and  $d_\eta{\colon} (\mdC \times \mcG)^2 {\to}\mdR_+$ be given by
\begin{equation}\label{eq: baseline metric appendix}
\mu(\HKoop) {\triangleq} \textstyle{\sum_{j \in [\ell]} }\frac{m_j}{m_{tot}}\delta_{(\lambda_j,\mcV_j)}\quad\text{ and }\quad d_\eta[(\lambda',\mcV'),(\lambda', \mcV')]{\triangleq}\eta |\lambda-\lambda'| {+} (1{-}\eta) \,d_\mcG(\mcV,\mcV'),
\end{equation}
with $|\cdot|$ applied on polar coordinates $\lambda,\lambda'$, $m_{tot} = \sum_{i \in [\ell]}m_i$, $\mcV_j$ the $m_j$-dimensional vector space in ${\rm HS}(\RKHS,\RKHS)$ spanned by the rank one operators of the right/left eigenfunctions associated with the eigenvalue $e^{\lambda_j}$ of $\HKoop$ (same notation for $\HKoop'$). Then, $(\mcS_r(\mcH),d_{\mcS})$ is a metric space, where $d_{\mcS}{\colon}\mcS_r(\mcH){\to}\mdR_+$ is given by
\begin{equation}\label{eq: spectral metric on operator appendix} 
      d_{\mcS}(\HKoop,\HKoop') =  W_{d_\eta,p}(\mu(\HKoop),\mu(\HKoop')).
   \end{equation}
\end{thm}

{\bf Discrete Optimal transport.}~For conciseness, we first recall discrete OT where one seeks a transport plan mapping samples from a source distribution to those of a target distribution while minimizing a transportation cost.
Formally, consider $\mcZ_S = \{z_i \in \mcZ \ | \ i \in [k_S]\}$ and $\mcZ_T =
\{z'_i \in \mcZ \ | \ i \in [k_T]\}$ as the sets of source and target samples in
a space $\mcZ$. We associate with these sets the probability distributions $\mu_S
= \sum_{i \in [k_S]}a_i \delta_{z_i}$ and $\mu_T = \sum_{i \in [k_T]}
b_i\delta_{z'_i}$  with $(\mba,\mbb) \in \Delta^{k_S} \times \Delta^{k_T}$ and
$\Delta^n = \{ \mbp \in \mdR_+^n\ | \ \sum_{i \in [n]}p_i = 1\}$ the $n$-simplex. Let
$\mbC \in \mdR_+^{k_S \times k_T}$ be the cost matrix with $C_{ij} =
c(z_i,z'_j)$ being the transport cost between $z_i$ and $z'_j$ given by the cost
function $c$. The Monge-Kantorovich problem aims at identifying a coupling
matrix, also denoted as OT plan $\mbP^* \in \mdR_+^{k_S \times k_T}$, that is solution of the
constrained linear problem:
\begin{equation}
   \label{eq: OT problem appendix}
   \min\limits_{\mbP \in \Pi(\mu_S,\mu_T)} \innerp{\mbC}{\mbP}_F \quad \text{s.t} \quad \Pi(\mu_S,\mu_T) = \{ \mbP \in \mdR_+^{k_S \times k_T} \ | \ \mbP\mb1 = \mba,~ \mbP^\intercal\mb1 = \mbb \}~,
\end{equation}
where $\Pi(\mba,\mbb)$ is the set of joint-distributions over $\mcZ_S \times
\mcZ_T$ with marginals $\mba$ and $\mbb$. In what follows, we denote
$L_c(\mu_S,\mu_T)$ the application returning the optimal value of problem (\ref{eq:
OT problem appendix}) where $c$ indicates the cost function. A fundamental property of OT is that, under suitable conditions on the cost function, the Wasserstein distance is a metric on the space of probability measures:  
\begin{thm}[Theorem 6.18 in \citet{villani2008optimal}]
\label{thm: wasserstein metric appendix}
Let $(\mcZ,d)$ be a separable complete metric space endowed with its Borel set. Let $p \in \mdN^*$, and $\mcP_p(\mcZ)$ the set of probability distributions on $\mcZ$ admitting moments of order $p$. Consider the application:
\begin{equation}
   W_p : (\mu,\nu) \in  \mcP(\mcZ) \times \mcP(\mcZ) \mapsto (L_{d^p}(\mu,\nu))^\frac{1}{p} \in \mdR_+~.
\end{equation}
Then, $(\mcP_p(\mcZ),W_p)$ defines a separable complete metric space, known as a Wasserstein space.
\end{thm}

{\bf Main proof.} For proof correctness, we restrict the Grassmann manifold on Hilbert-Schmidt operators $\mcG$ to the set of operators with rank at most $r$, denoted by $\mcG_r$. This restricted space endowed with the Hilbert-Schmidt norm is a complete metric space as detailed in \Cref{appendix: Grassman}. The next two propositions detail the essential building blocks to derive a Wasserstein metric on $\mcS_r(\mcH)$. \Cref{prp: inclusion map to probability space} specifies an inclusion map from $\mcS_r(\mcH)$ to a space of probability measures while \cref{prp: metric projector eigenvalue} defines a metric on the measures' support space with sufficient topological properties to derive a Wassertein metric on $\mcS_r(\mcH)$.

\begin{prp}
\label{prp: inclusion map to probability space}
Consider $p \in \mdN^*$ and the embedding map:
\begin{equation}
   \mu : T \in S_r(\mcH) \mapsto \sum_{i \in [l]} \frac{l_i}{l_{tot}}\delta_{(\lambda_i,\mcV_i)} \in \mcP_p(\mdC \times \mcG_r)~,
\end{equation}
where $\mcV_j$ the $m_j$-dimensional vector space in ${\rm HS}(\RKHS,\RKHS)$ spanned by the rank one operators of the right/left eigenfunctions associated with the eigenvalue $e^{\lambda_j}$ of $\HKoop$, and $m_{tot} = \sum_{i \in [l]} m_i$. Then, $L$ is a one-to-one inclusion map. 
\end{prp}
\begin{proof}
   Let $T \neq T' \in \mcS_r(\mcH)$, both operators differ by at least one pair $(\lambda_i,\mcV_i) \in \mcG_r$, by symmetry the pair is associated to $T$. Since $(\mdC,|.|)$ and $(\mcG_r, d_\mcG)$ are metric spaces, the singleton $\{(\lambda_i,\mcV_i)\}$ belongs to the Borel set. Therefore, $\mu_T((\lambda_i,\mcV_i)) = m_i/m_{tot}$ while $\mu_{T'}((\lambda_i,\mcV_i)) = 0$, i.e $\mu_T \neq \mu_{T'}$. 
\end{proof}

\begin{prp}
\label{prp: metric projector eigenvalue}
Consider $\eta \in (0,1)$, $\omega_{ref} \in \mdR_+^*$, and the application:
\begin{equation}
   \label{eq: baseline metric appendix proposition}
   d_\eta : ((\lambda,\mcV),(\lambda', \mcV')) \in (\mdC \times \mcG_r)^2 \mapsto \eta |\lambda-\lambda'| + (1-\eta) d_\mcG(\mcV,\mcV') \in \mdR_+~.
\end{equation}
Then, $(\mdC \times \mcG_r, d_\eta)$ is a separable complete metric space.
\end{prp}
\begin{proof}
   By proposition~\ref{prp: grassman metric}, $(\mcG_r,d_\mcG)$ is a separable complete metric space. Hence, for any $\eta \in (0,1)$, $(\mdC \times \mcG_r,d_\eta)$ is a separable complete metric space as $(\mdC, d_{val})$ is homeomorphic to $(\mdC, |.|)$.
\end{proof}
Note that we introduce a metric, $d_{val}$, that compares Koopman modes' eigenvalues from physics-informed quantities, namely the time-scales $\rho$ and the ocsillating frequencies $\omega$. The previous two propositions lead to our main contribution, a Wasserstein metric on the space of non-defective finite rank operators  $\mcS_r(H)$:
\begin{prp}
   \label{prp: spectral metric on operator}
   Consider $\eta \in (0,1)$, $p \in \mdN^*$, and the application: 
   \begin{equation}
      d_{\mcS}: (T,T') \in \mcS_r(\mcH) \times \mcS_r(\mcH) \mapsto  W_{d_\eta,p}(\mu(T),\mu(T')) \in \mdR_+~.
   \end{equation}
   Then, $(\mcS_r(\mcH),d_{\mcS})$ is a metric space.
\end{prp}

\begin{proof}
   Application of \cref{thm: wasserstein metric appendix} with \Cref{prp: metric projector eigenvalue,prp: inclusion map to probability space}. 
\end{proof}

\subsection{Grassman metric}
\label{appendix: Grassman}

A Grassmann manifold is a collection of vector subspaces of a given vector space. Such manifolds appear in a handful of applications whenever subspaces must be compared. The particular case of Grassman manifolds gathering all equidimensional subspaces of a finite-dimensional real vector space has been extensively studied, see \citet{bendokat2024grassmann} for a thorough review. In our context, this particular setting is limiting as we must consider a manifold including subspaces of various dimensions over a possibly infinite-dimensional complex vector space. On such manifolds, a classical metric compares subspaces through the associated orthogonal projectors with the operator norm \citep{andruchow2014grassmann}. Unfortunately, this metric is computationally expensive, and the topology it induces does not provide the necessary conditions to derive Wasserstein metrics, namely, the separability. In the following proposition, we define a Grassmann manifold with the necessary conditions to derive Wasserstein metrics. 

\begin{prp}
\label{prp: grassman metric}
   Let $r \in \mdN^*$ be fixed, and $\mcG_r(\mcH)$ denote the set of all closed vector subspaces of a (possibly infinite-dimensional) separable Hilbert space $\mcH$ having dimension at most $r$. Endow $\mcG_r(\mcH)$ with the well-defined metric: 
   \begin{equation}
      d_\mcG: (\mcU,\mcV) \in \mcG_r(\mcH) \times \mcG_r(\mcH) \mapsto \| P_\mcU - P_\mcV \|_{\mcH\mcS} \in \mdR_+~,
   \end{equation}
   where $P_\mcU$ is the orthogonal projector onto $\mcU$, and $\|.\|_{\mcH\mcS}$ is the Hilbert-Schmidt norm. Then $(\mcG_r(\mcH),d_\mcG)$ is a separable complete metric space.
\end{prp}

\begin{proof}

   Before the main proof, we investigate the properties of an inclusion map, which is useful for determining the metric and completeness properties. 

   \begin{lmm}
      \label{lmm: inclusion map}
      The map $i:\mcV \in \mcG_r(\mcH) \mapsto P_\mcV \in \mcH\mcS(\mcH)$, which associates to any subspace the orthogonal projector onto itself, is well defined and a one-to-one inclusion.
   \end{lmm}
   \begin{proof}
      Since any $\mcV \in \mcG_r(\mcH)$ is finite dimensional, it is closed, and the orthogonal projector $P_\mcV$ is a well-defined bounded linear operator by consequence of the Hilbert projection theorem. Furthermore, since $\mcH$ is separable, it admits an orthogonal basis, respecting the orthogonal decomposition $\mcH = \mcV \oplus \mcV^\perp$. Since $\dim(\mcV)\leq r$ and by invariance of the Hilbert-Schmidt norm to change of basis, $\|P_\mcV\|_{\mcH\mcS}$ is finite, more precisely: $\|P_\mcV\|_{\mcH\mcS}^2 = \dim(\mcV) < r$. Furthermore, for any $\mcV \neq \mcV' \in \mcG_r(\mcH)$, $P_\mcV \neq P_{\mcV'}$ due to the orthogonal decomposition $\mcH = \mcV \cap \mcV' \oplus \mcV/(\mcV \cap \mcV') \oplus \mcV'/(\mcV \cap \mcV') \oplus (\mcV \cup \mcV')^\perp$.
   \end{proof}

   \begin{lmm}
      \label{lmm: projector completness}
      $\mcP_r = \{P_\mcV \ | \ \mcV \in \mcG_r(\mcH) \}$ is a closed subspace of $\mcH\mcS(\mcH)$ for the topology induced by the Hilbert-Schmidt norm.
   \end{lmm}

   \begin{proof}
      First notice that, $\mcP_r \subset \mcH\mcS(\mcH)$ and $\mcH\mcS(\mcH)$ is a Hilbert space, thus complete. Consider a sequence $(P_n)_{n \in \mdN} \in \mcP_r$ converging to an element $P \in \mcH\mcS(\mcH)$ (i.e. $\|Pn-p\|_{\mcH\mcS} \to 0 $), let's prove that $P \in \mcP_r$. 
      
      Since $P \in \mcH\mcS(\mcH)$, it follows that the adjoint operator $P^* \in \mcH\mcS(\mcH)$ exists and since $\|P^* - P_n^*\|_{\mcH\mcS} = \|P -P_n\|_{\mcH\mcS} \to 0$, the operator $P$ is self-adjoint $P = P^*$.
      Furthermore by composition $P^2 \in \mcH\mcS(\mcH)$, and: 
      \begin{align}
            \|P^2 - P\|_{\mcH\mcS} & \leq \|P^2 - P_n^2\|_{\mcH\mcS} + \|P_n^2 - P_n\|_{\mcH\mcS} + \|P_n - P\|_{\mcH\mcS} \\ 
            & \leq \|P^2 - P_n^2\|_{\mcH\mcS} + \|P_n - P\|_{\mcH\mcS} \\ 
            & \leq \|P_n - P\|_{\mcH\mcS} (1 + \|P\|_{\mcH\mcS} + \|P_n\|_{\mcH\mcS}) 
      \end{align}
      Since $ \|P_n - P\|_{\mcH\mcS} \to 0$, it follows that $P^2 = P$, meaning that $P$ is an orthogonal projector. Let $\mcV$ denote the closed vector subspace associated to $P$. Since $P$ is an orthogonal projector with a finite Hilbert-Schmidt norm, $\mcV$ is finite dimentional, and $\dim(\mcV) =  \|P\|_{\mcH\mcS}^2 \leq r$, as $\|P_n\|_{\mcH\mcS}^2 \leq r$ for any $n \in \mdN$. Thus $P \in \mcP_r$, indicating that $\mcP_r$ is a closed subset of $\mcH\mcS(\mcH)$.
   \end{proof}

   \textbf{Main proof.} Since the map $i$, defined in Lemma~\ref{lmm: inclusion map}, is a one-to-one inclusion into the space $\mcH\mcS(\mcH)$, the metric derived from the Hilbert-Schmidt norm ($\|.\|_{\mcH\mcS}$) induces a metric onto the space $\mcG_r(\mcH)$. Furthermore, since $\mcP_r = \{P_\mcV \ | \ \mcV \in \mcG_r(\mcH) \}$ is a closed subset of a complete space by Lemma~\ref{lmm: projector completness}, it is complete. Hence, the metric space $(\mcG_r(\mcH),d_\mcG)$ is complete. Lastly, the space $\mcH\mcS(\mcH)$ is separable as it is homeomorphic to $\mcH \otimes \mcH$, which is a separable space as the tensor product of the separable space $\mcH$. Hence $\mcP_r \subset \mcH\mcS(\mcH)$ is also separable by inclusion. Finally, the metric space $(\mcG_r(\mcH),d_\mcG)$ is separable and complete, which concludes the proof.
   
\end{proof}

\section{Spectral Grassman barycenter}
\label{appendix: barycenter}

\subsection{Problem formulation}
\label{appendix: barycenter theory}
Computing barycenters is a fundamental problem for many unsupervised methods. When data lie in a metric space, it is known as the \emph{Fréchet mean problem}. It involves identifying an element that minimizes a weighted sum of distances to the observations. Formally, given the importance weights $\bs\gamma {\in} \Delta^N$, assuming~\ref{enum:assumption1}-\ref{enum:assumption3}, for $p{=}2$ in \Cref{thm:main_wasserstein_metric} we aim to solve:
\begin{equation}
   \label{eq:frechet_mean_problem appendix}
   \argmin_{T \in \mcS_r(\mcH)} \sum_{k \in [N]} \gamma_i d_\mcS(\HKoop,\HTO{k})^2, 
\end{equation} 
By construction of $d_\mcS$, problem~\ref{eq:frechet_mean_problem} corresponds to the estimation of Wasserstein barycenter over a set of finite measures with support on a manifold embedded in a (possibly infinite-dimensional) Hilbert space. From a theoretical standpoint, the existence (and uniqueness) of Wasserstein barycenters has been established in several settings, including continuous measures~\citep{agueh2011barycenters}, and discrete measures on finite-dimensional Euclidean spaces~\citep{anderes2016discrete}, and measures on geodesic spaces~\citep{le2017existence}. In \citet{han2024geometry}, the authors address the case of continuous measures on infinite-dimensional metric spaces. In our settings, we assume the existence of a barycenter in the closure of $\mcS_r(\mcH)$, see discussion on the extension to general operators in \cref{section: method}. A formal proof would require extending previous works to finite measures on manifolds in infinite-dimensional Hilbert spaces.

From a computational standpoint, problem~\ref{eq:frechet_mean_problem appendix} is closely related to the \emph{free-support Wassertein barycenter} estimation, which aims at optimizing the support and, optionally, the mass of the atoms parametrizing the barycenter. State-of-the-art algorithms typically rely on a coordinate descent scheme \citep{wright2015coordinate}, alternating between transport plan computation and measure optimization with strategies including gradient descent \citep{cuturi2014fast}, fixed point iteration \citep{alvarez2016fixed,lindheim2023simple}, stochastic optimization \citep{claici2018stochastic,li2020continuous}, or proximal operators \citep{qian2021inexact}. In our context, on the measure's support, i.e., the barycenter's spectral decomposition, must be optimized as the eigensubspaces' dimensions condition the masses according to the embedding map in \cref{eq: baseline metric appendix}.

\paragraph{A parametric problem formulation.} Whenever the RKHS $\mcH$ is infinite dimensional (the finite case is discussed in  \cref{appendix: barycenter finite dimensional rkhs}), the Fréchet mean problem (\cref{eq:frechet_mean_problem appendix}) is intractable in its original form. We restrained the optimization over a set of parametrized operators defined such that for any $\bs\theta \triangleq (\bs\lambda,\bs\alpha,\bs\beta,\mbx)$: 
\begin{equation}
   \label{eq: parametrized operator appendix}
   T_{\bs\theta} : h \in \mcH \mapsto  \textstyle{\sum_{i \in [r]}} \lambda_i \innerp{ \kappa\bs\alpha_i}{h}_\mcH \kappa\bs\beta_i \in \mcH
\end{equation}

where $\bs\lambda \in \mdC^r$, $\mbx \in \mcX^n$ are state space control points, and $\bs\alpha,\bs\beta \in \mdC^{n \times r}$ control parameters acting on the representer functions $\kappa_\mbx = \{\kappa(.,x_j)\}_{j \in [n]}$ with $\kappa$ the kernel of $\mcH$, i.e. $\kappa_{\mbx}\bs\alpha_i \triangleq \sum_{j \in [n]}\kappa_{x_j}\alpha_{ji}$. While these operators are compact with rank at most $r$, further constraints on the control points and parameters are required to ensure a spectral decomposition (see \Cref{eq: modal decomposition}). Together with the definition of discrete optimal transport (see \Cref{sec:background}), it leads to the constrained optimization problem:
\begin{equation}
   \label{eq: constrained spectral barycenter problem appendix}
   \argmin\limits_{\bs\theta, \mbP} \sum_{i \in [N]} \gamma_i \innerp{\mbC_i(\bs\theta)}{\mbP_i}_F \quad \textrm{s.t.} \quad \left\{
   \begin{array}{ll}
    \bs \alpha^* \mbK \bs\beta = \mbI & \mbK = \{\kappa(x_i,x_j)\}_{(i,j) \in [n]^2} \\
    \bs\beta_j^* \mbK \bs\beta_j = 1,~ \forall j \in [r] & \mbP_i \in \Pi(\mu_{T_{\bs{\theta}}},\mu_{T_i}),~\forall i \in [N]
   \end{array}\right.
\end{equation}
where $\mbP = \{\mbP_i\}_{i \in [N]}$, $\widehat{\bs{T}} = \{\widehat{T}_i\}_{i \in [N]}$, such that $(\mbC_i(\bs\theta), \mbP_i)$ are the cost and transport matrices associated to the Wasserstein metric, $d_\mcS$ defined in proposition~\ref{prp: spectral metric on operator}, between the parametric operator $T_{\bs\theta}$ and $\widehat{T}_i$.

\subsection{Barycenter estimation method}
\label{appendix: barycenter estimation}

\paragraph{An inexact coordinate descent scheme.} In what follows, let $\mcX$ be a bounded open set of $\mdR^d$ with $d \in \mdN^*$ and $k : \mcX \times \mcX \to \mdR$ be a differentiable kernel.
Following \citet{cuturi2014fast}, we propose an inexact coordinate descent scheme with a cyclic update rule designed to converge to a stationary point of problem~\ref{eq: constrained spectral barycenter problem appendix}. Each cycle begins with the computation of the exact optimal transport plans $\mbP$ to enforce sparsity. This step is carried out with the algorithm of \citet{bonneel2011displacement}, whose complexity depends on the number of eigenvalues, typically small in practice \citep{brunton2021modern}. The subsequent coordinate updates are performed using a few gradient descent steps with a first-order optimizer such as \textsc{Adam}~\citep{Kingma2014AdamAM}. It starts with the eigenvalues $\bs{\lambda}$, optionally followed by the state spaces control points $\mbx$, for which no optimization constraints exist. Next, the right eigenfunctions, $\bs\beta$, are updated only considering the normalization constraints: $\bs{\beta_j}^*\mbK\bs\beta_j = 1$, $j \in [r]$. Finally, the left eigenfunctions, $\bs\alpha$, are updated considering the affine constraints: $\bs\alpha^*\mbK\bs\beta = \mbI$, leading to an iterated closed-form projection scheme detailed in \Cref{eq: projection alpha}. \Cref{alg: main algorithm} summarizes the full procedure, and further implementation details are provided in the next paragraphs. We usually repeat 10 gradient descent steps in experiments when updating $\bs\lambda,\bs\alpha,\bs\beta$ and $\mbx$ at each cycle.

\begin{algorithm}
\begin{algorithmic}[1]
\Require $\widehat{\bs{T}} \triangleq \{\widehat{T}_i\}_{i \in [N]} \in \mcS_r(\mcH)^N$, 
\State $\bs\theta \triangleq(\bs\lambda, \bs\alpha, \bs\beta, \mbx) \gets Initialization(\widehat{\bs{T}})$
\Comment{Operator parameters, see \cref{eq: parametrized operator}}
\While{not converged}
\State $\mbP \gets ComputeTransportPlans(T_{\bs\theta},\widehat{\bs{T}})$
\Comment{See \Cref{thm:main_wasserstein_metric}, and \Cref{sec:background}}
\State $\bs\lambda \gets UpdateEigenValues(\bs\theta,\mbP,\widehat{\bs{T}})$ 
\If{optimize control points}
\State $\mbx \gets UpdateControlPoints(\bs\theta,\mbP,\widehat{\bs{T}})$
\EndIf
\State $\bs\beta \gets UpdateRightEigenFunctions(\bs\theta,\mbP,\widehat{\bs{T}})$ 
\Comment{Detailed in \Cref{alg: Update right eigenfunctions}}
\State $\bs\alpha \gets UpdateLeftEigenFunctions(\bs\theta,\mbP,\widehat{\bs{T}})$ 
\Comment{Detailed in \Cref{alg: Update left eigenfunctions}}
\EndWhile
\State $\mbP \gets ComputeTransportPlans(T_{\bs\theta},\widehat{\bs{T}})$
\\
\Return $\bs\theta, \mbP$
\end{algorithmic}
\caption{Spectral Barycenter}
\label{alg: main algorithm}
\end{algorithm}

\paragraph{Update right eigenfunctions.}
We detail the \textit{UpdateRightEigenFunctions} step of \Cref{alg: main algorithm}. Let $\bs\lambda,\mbx,\bs\alpha$ and $\mbP$ be fixed; we aim to perform minimization steps of problem~\ref{eq: constrained spectral barycenter problem appendix} with regard to $\bs\beta$, the parameters controlling the right eigenfunctions. Each optimization step consists of a first-order gradient descent step followed by a projection of each eigenfunction on the RKHS unit sphere as described in \Cref{alg: Update right eigenfunctions}. 
\begin{algorithm}
   \begin{algorithmic}[1]
      \Require $\bs\theta = (\bs\lambda,\mbx,\bs\alpha,\bs\beta), \mbP, \widehat{\bs{T}}$
      \While{stopping criteria not met}
      \State $\widehat{\bs{\beta}} \gets $ Gradient descent step of $J(\bs\theta,\mbP;\widehat{\bs{T}})$ w.r.t $\bs\beta$.
      \For{$i \in [r]$}
      \Comment{$r$ being the number of eigenfunctions}
      \State $\bs\beta_i \gets \widehat{\bs{\beta}}_i/\sqrt{\widehat{\bs{\beta}}_i\mbK\widehat{\bs{\beta}}_i}$
      \Comment{Projection on the RKHS unit sphere}
      \EndFor
      \EndWhile
      \Return $\bs\beta$
   \end{algorithmic}
   \caption{UpdateRightEigenFunctions}
   \label{alg: Update right eigenfunctions}
\end{algorithm}

\paragraph{Update left eigenfunctions.}
We detail the \textit{UpdateLeftEigenFunctions} step of \Cref{alg: main algorithm}. Let $\bs\lambda,\mbx,\bs\beta$ and $\mbP$ be fixed; we aim to perform minimization steps of problem~\ref{eq: constrained spectral barycenter problem appendix} with regard to $\bs\alpha$, the parameters controlling the left eigenfunctions. Each optimization step consists of a first-order gradient descent step followed by a projection onto the manifold induced by the spectral decomposition constraint: $\bs\alpha^*\mbK\beta=\mbI$. Algorithm~\ref{alg: Update left eigenfunctions} describes the optimization procedure, and the next paragraph discusses the projection.

\begin{algorithm}
   \begin{algorithmic}[1]
      \Require $\bs\theta = (\bs\lambda,\mbx,\bs\alpha,\bs\beta), \mbP, \widehat{\bs{T}}$
      \While{stopping criteria not met}
      \State $\widehat{\bs{\alpha}} \gets $ Gradient descent step of $J(\bs\theta,\mbP;\widehat{\bs{T}})$ w.r.t $\bs\alpha$.
      \State $\bs\alpha \gets \widehat{\bs{\alpha}} - \bs\beta\left(\left(\widehat{\bs{\alpha}}^*\mbK\bs\beta - \mbI\right)\left(\bs\beta^*\mbK\bs\beta\right)^{-1}\right)^*$
      \Comment{Manifold Projection, see below.}
      \EndWhile
      \Return $\bs\alpha$
   \end{algorithmic}
   \caption{UpdateLeftEigenFunctions}
   \label{alg: Update left eigenfunctions}
\end{algorithm}

\paragraph{Projection step.}
Let $\widehat{\bs{\alpha}} \in \mdC^{n \times r}$ be the estimated parameters controlling the left eigenfunctions after a gradient descent step without constraints. Hence, these parameters might not verify the spectral decomposition constraint. We aim to identify the closest parameters, $proj(\widehat{\bs{\alpha}}) \in \mdC^{n \times r}$, for the RKHS metric and lying on the manifold induced by the spectral decomposition constraint. It leads to a constrained optimization problem:
\begin{equation}
   \begin{array}{rcrl}
     Proj(\widehat{\bs{\alpha}})&  \triangleq & \argmin\limits_{\bs\alpha} & \Tr\left((\bs\alpha-\widehat{\bs{\alpha}})^*\mbK(\bs\alpha-\widehat{\bs{\alpha}})\right) \\
     && \textrm{s.t.}& \bs\alpha^*\mbK\bs\beta = \mbI
   \end{array}~.
\end{equation}
Considering the real representation of the problem, it becomes a convex problem, and since $r \ll n$, it is strictly feasible. Strong duality holds by Slater's constraint qualification. As the optimization function $J$ and constraints are differentiable with respect to $\alpha$, the KKT conditions are necessary and sufficient conditions to characterize the optimum. The Lagrangian can be expressed as: 
\begin{equation}
   L(\bs\alpha,\bs\mu,\bs\nu) \triangleq \Tr\left((\bs\alpha-\widehat{\bs{\alpha}})^*\mbK(\bs\alpha-\widehat{\bs{\alpha}})\right) + \bs\mu^\intercal(\Rel(\bs\alpha^*\mbK\bs\beta)-\mbI) + \bs\nu^\intercal(\Img(\bs\alpha^*\mbK\bs\beta))~.
\end{equation}
Taking Wirtinger derivative notation, the optimal primal-dual variables verify: 
\begin{equation}
   \left\{\begin{array}{l}
      \nabla_{\overline{\bs\alpha}} L(\bs\alpha,\bs\mu,\bs\nu) = \mbK(\bs\alpha - \widehat{\bs{\alpha}} + \bs\beta(\bs\mu - i \bs\nu)^\intercal) = \mb0\\
      \nabla_{\bs\mu} L(\bs\alpha,\bs\mu,\bs\nu) = \Rel(\bs\alpha^*\mbK\bs\beta)-\mbI = \mb0 \\
      \nabla_{\bs\nu} L(\bs\alpha,\bs\mu,\bs\nu) =  \Img(\bs\alpha^*\mbK\bs\beta) = \mb0
   \end{array}\right.
\end{equation}
Regarless of the rank of $\mbK$, $\bs\alpha - \widehat{\bs{\alpha}} + \bs\beta(\bs\mu - i \bs\nu)^\intercal = \mb0$ always verifies the first optimality equation. It leads to the projector:
\begin{equation}
\label{eq: projection alpha}
   Proj(\widehat{\bs{\alpha}}) = \widehat{\bs{\alpha}} - \bs\beta\left(\left(\widehat{\bs{\alpha}}^*\mbK\bs\beta - \mbI\right)\left(\bs\beta^*\mbK\bs\beta\right)^{-1}\right)^*~.
\end{equation}

\subsection{Case of finite-dimensional RKHS}
\label{appendix: barycenter finite dimensional rkhs}
Consider $\mcH$ be finite $d$-dimensional RKHS with the orthonormal basis $\mbf = \{f_i\}_{i \in [d]}$. For instance, $\mcH$ is based on a functional dictionary as used in extended DMD \citep{kutz2016dynamic}. Let $\mcS_r(\mcH)$ be the set of non-defective compact operators acting on $\mcH$ with rank at most $r \leq d$. We aim to solve: 
\begin{equation}
   \argmin_{T \in \mcZ_r(\mcH)} \sum_{i \in [N]} \gamma_i d_\mcS(T,\widehat{T_i})^2~, 
\end{equation}
with $\bs\gamma \in \Delta^N$ the importance weights, $\{\widehat{T}_i \in \mcZ_r(\mcH) \ | \ i \in [N] \}$ estimated operators, and $d_\mcS$ defined in \Cref{thm:main_wasserstein_metric} given $\eta \in (0,1)$ and $p=2$. For any compact operator acting on $\mcH$ with rank at most $r$, there exists coefficients $\bs\lambda \in \mdC^r$, and control parameters of the functional basis $\bs\alpha,\bs\beta \in \mdC^{d \times r}$ such that: 
\begin{equation}
   T_\theta : h \in \mcH \mapsto \sum_{i \in [r]} \lambda_i \innerp{f_{\bs\alpha_i}}{h}_\mcH f_{\bs\beta_i} \in \mcH~, 
\end{equation}
where $f_{\bs\alpha_i} \triangleq \sum_{j \in d} \alpha_{ji}f_j$, $f_{\bs\beta_i} \triangleq \sum_{j \in d} \beta_{ji}f_j$, and $\theta \triangleq (\bs\lambda,\bs\alpha,\bs\beta)$. To ensure non-defectiveness of $T_\theta$ further constraints are imposed on the control parameters, which leads to the constrained optimization problem: 
\begin{equation}
   \label{eq: constrained spectral barycenter finite case}
   \begin{array}{rl}
      \argmin\limits_{\bs\lambda, \bs\alpha ,\bs\beta} & \sum_{i \in [N]}\gamma_i d_\mcS(T_\theta,\widehat{T_i})^2 \\ 
    \textrm{s.t.} & \bs\alpha^*\bs\beta = \mbI \\ 
   & \bs\beta_i^* \bs\beta_i = 1, \quad \forall i \in [r]
   \end{array}
\end{equation}
Note that this optimization problem is related to the infinite-dimensional problem defined \cref{eq: constrained spectral barycenter problem} by assuming the control points to be fixed such that the kernel matrix is the identity matrix, i.e. $\mbK=\mbI$. It follows that the optimization procedure described in the case of infinite-dimensional RKHS in \Cref{appendix: barycenter estimation} also handles the finite-dimensional case.

\section{Proofs of Statistic results}
\label{appendix: stat_bound}

We now prove the main statistical results in this section.

\prpstatbound*

\begin{proof}[Proof of Theorem \ref{thm:stat_bound}.]

Without loss of generality, we can assume that the operators eigenvalues are of multiplicity $1$. 
Then the discrete distribution representation of the operator $\HTO{k}$ provided in \eqref{eq: baseline metric} becomes
\begin{equation}
\mu(\HTO{k}) {\triangleq} \frac{1}{r_k} \textstyle{\sum_{j \in [r_k]} }\delta_{(\lambda_j(k),\mcV_j(k))}.
\end{equation}
Similarly
\begin{equation}
\mu(\ETO{k}) {\triangleq} \frac{1}{r_k} \textstyle{\sum_{j \in [r_k]} }\delta_{(\widehat{\lambda}_j(k),\widehat{\mcV}_j(k))}.
\end{equation}

\paragraph{Stability of the $d_\mcS$ metric.}~ Next by definition of $d_{\mcS}$ in \eqref{eq: spectral metric on operator} and the triangular inequality applied to the Wasserstein metric:
\begin{equation}
\label{eq:Wasserstein-perturbation}
\big| d_\mcS(\ETO{1},\ETO{2}) - d_\mcS(\HTO{1},\HTO{2})  \big| 
\;\le\; W_p\big(\mu(\HTO{1}), \mu(\ETO{1})\big) + W_p\big(\mu(\HTO{2}), \mu(\ETO{2})\big).
\end{equation}
We recall that $W_p\big(\mu(\HTO{k}), \mu(\ETO{k})\big)$ is defined as:
\begin{equation*}
\label{eq:discrete-Wasserstein0}
W_p\big(\mu(\HTO{k}), \mu(\ETO{k})\big) := 
\Bigg(
\min_{\mbP \in \Pi_{\rm uniform}(r_k)} 
\sum_{i=1}^{r_k} \sum_{j=1}^{r_k} c_{i,j}^p \, P_{i,j} 
\Bigg)^{1/p},
\end{equation*}
where the cost matrix $C_k = (c_{i,j})_{i,j\in [r_k]}$ is defined as
$$
c_{i,j}:= d_\eta((\lambda_i(k),\mcV_i(k)),(\widehat{\lambda}_j(k),\widehat{\mcV}_j(k)))\geq 0,\quad \forall i,j\in [r_k],
$$
and the set of uniform transport plans is:
\begin{equation*}
\Pi_{\rm uniform}(r_k) := 
\Big\{ \mbP \in \mathbb{R}_+^{r_k \times r_k} \ \Big|\ 
\mbP \mathbf{1} = \frac{1}{r_k}\mathbf{1}, \ \mbP^\intercal \mathbf{1} = \frac{1}{r_k}\mathbf{1} \Big\}.
\end{equation*}
Then we note that the transport plan $\pi_{i,i} = 1/r_k$ for any $i\in [r_k]$ and $\pi_{i,j} = 0$ for any $i\neq j$ belongs to the set $\Pi_{\rm uniform}(r_k)$. Consequently
\begin{equation}
\label{eq:discrete-Wasserstein}
W_p^p\big(\mu(\HTO{k}), \mu(\ETO{k})\big) \leq 
\sum_{i=1}^{r_k} \sum_{j=1}^{r_k} c_{i,j}^p \, \pi_{i,j}  = \frac{1}{r_k}\sum_{i=1}^{r_k} c_{i,i}^p
,
\end{equation}

In view of \eqref{eq:entriwise_error_cost_matrix2} below, we prove w.p.a.l. $1-\delta$ that 
$$
\max_{i\in [r_k]}\{ c_{i,i}\} \lesssim \varepsilon_n(\delta) :=  n^{-\frac{\rpar-1}{2(\rpar+\spar)}}\ln(2\delta^{-1}),\quad\forall k\in [2].
$$
Then we get with the same probability 
$$
W_p\big(\mu(\HTO{k}),\mu(\ETO{k})\big) \lesssim n^{-\frac{\rpar-1}{2(\rpar+\spar)}}\ln(2\delta^{-1}),\quad \forall k\in [2],
$$
and consequently we obtain the final bound on $\big| d_\mcS(\ETO{1},\ETO{2}) - d_\mcS(\HTO{1},\HTO{2})  \big|$ in view of \eqref{eq:Wasserstein-perturbation}.

\paragraph{Bounding the learning error $\|\HTO{} - \ETO{}\|$.}~ For brevity we set $\|\cdot\|$ for the operator norm $\|\cdot \|_{\RKHS \to \RKHS}$ and $\|\cdot\|_{\RKHS}$ for the Hilbert-Schmidt norm $\|\cdot \|_{\mathrm{HS}(\RKHS,\RKHS)}$. For any $k$, we introduce the population RRR and Ridge operators as 
$$
\HTO{k,\reg} :=(\Cx^k{+}\gamma I)^{-\frac{1}{2}}\SVDrr{(\Cx^k {+}\gamma I)^{-\frac{1}{2}}\Cxy^k}{r_k},\quad \HTO{k,\reg}^R = (\Cx^k+\reg)^{-1} \Cxy^k
.
$$
\\
Then we have the following Bias-Variance decomposition in operator norm:
\begin{align}
\label{eq:Bias_Variance_dec}
    \|\HTO{k} - \ETO{k}\| &\leq \underbrace{\|\HTO{k} - \HTO{k,\reg}\|}_{\text{$=:a_1$ ``Bias'' } } + \underbrace{\|\HTO{k,\reg} - \ETO{k}\|}_{\text{ $=:a_2$ ``Variance"} }.
\end{align}

\paragraph{Bias term $a_1$.}~We have
\begin{align}
\label{eq:Bias_dec1}
 a_1 \leq \Vert\HTO{k}-\HTO{k,\reg}^R\Vert + \Vert\HTO{k,\reg}^R - \HTO{k,\reg}\Vert.
\end{align}

Next since $P_{\leq r_k} \TO_{k}S_{\im_k} = S_{\im_k} %
\HTO{k}$, we get
$$
S^*_{\im_k} P_{\leq r_k}  \TO_{k} S_{\im_k} = \Cx^{k}\,\HTO{k},
$$
and
\begin{align*}
\HTO{k,\reg}^R&=(\Cx^{k}+\reg I)^{-1}\Cx^{k}\HTO{k} + (\Cx^{k}+\reg I)^{-1} S^*_{\im_k} P_{\leq r_k}^\bot\TO_{\im_k} S_{\im_k} \\
&= \HTO{k}-\reg(\Cx^{k}+\reg I)^{-1}(\Cx^{k})^{(\alpha-1)/2}(\Cx^{k})^{\dagger(\alpha-1)/2}\HTO{k} + (\Cx^{k}+\reg I)^{-1} S^*_{\im_k} P_{\leq r_k}^\bot\TO_{\im_k} S_{\im_k}.
\end{align*}

Therefore, 
\begin{align}
\label{eq:Bias_dec2}
\Vert\HTO{k}-\HTO{k,\reg}^R\Vert &\leq \reg\Vert(\Cx^{k}+\reg I)^{-1}(\Cx^{k})^{(\alpha-1)/2}\Vert\,\Vert[(\Cx^{k})^\dagger]^{1-\alpha}\HTO{k}\Vert + \frac{1}{\sqrt{\reg}}\Vert P_{\leq r_k}^\bot\TO_{\im_k}\Vert \sqrt{c_{\RKHS}} \notag\\
&\leq \reg^{(\alpha-1)/2}\Vert[(\Cx^{k})^\dagger]^{1-\alpha}\HTO{k}\Vert+\sqrt{\frac{c_{\RKHS}}{\reg}}\Vert P_{\leq r_k}^\bot\TO_{\im_k}\Vert 
\end{align}

We tackle now the second term in the right-hand side of \eqref{eq:Bias_dec1}. We note first that
\begin{align*}
    \HTO{k,\reg}^R - \HTO{k,\reg}&= (\Cx^{k}+\reg I)^{-1} \Cxy^{k} (I-\Pi_{r_k}).
\end{align*}
Hence we get
\begin{align}
\label{eq:Bias_dec3}
\Vert \HTO{k,\reg}^R - \HTO{k,\reg}\Vert\leq \frac{1}{\sqrt{\reg}}\sigma_{r_k+1}((\Cx^{k}+\reg I)^{-1/2} \Cxy^{k})\leq \frac{\sqrt{c_{\RKHS}}}{\sqrt{\reg}}\Vert P_{\leq r_k}^\bot\TO_{\im_k}\Vert.
\end{align}

Combining \eqref{eq:Bias_dec1}, \eqref{eq:Bias_dec2} and \eqref{eq:Bias_dec3} with the assumption $\Vert|[(\Cx^k)^\dagger]^{\frac{\alpha-1}{2}}\HTO{k}\Vert_{\RKHS{\to}\RKHS}{<}\infty$, we obtain the following control on the bias:
\begin{align}
\label{eq:Bias_dec4}
    a_1 &= \|\HTO{k} - \HTO{k,\reg}\| \leq \reg^{(\alpha-1)/2}\Vert[(\Cx^{k})^\dagger]^{1-\alpha}\HTO{k}\Vert+\frac{2\,\sqrt{c_{\RKHS}}}{\sqrt{\reg}}\Vert P_{\leq r_k}^\bot\TO_{\im_k}\Vert\notag\\
    &\lesssim \reg^{(\alpha-1)/2}+\frac{2\,\sqrt{c_{\RKHS}}}{\sqrt{\reg}}\Vert P_{\leq r_k}^\bot\TO_{\im_k}\Vert
\end{align}

\paragraph{Variance term $a_2$.}~ We note first 
\begin{align*}
    \HTO{k,\reg} - \ETO{k} &= (\Cx^k + \reg I)^{-1/2}(\Cx^k + \reg I)^{1/2}\big( \HTO{k,\reg} - \ETO{k}\big)
\end{align*}
Taking the operator norm, we get
\begin{align*}
    a_2 = \|\HTO{k,\reg} - \ETO{k}\| &\leq  \|\Cx^k + \reg I\|^{-1/2} \|
    (\Cx^k + \reg I)^{1/2}\big( \HTO{k,\reg} - \ETO{k}\big)\| \leq \frac{1}{\sqrt{\reg}} \|(\Cx^k + \reg I)^{1/2}\big( \HTO{k,\reg} - \ETO{k}\big)\|.
\end{align*}

Define $B_k:=(\Cx^k+\reg I)^{-1/2}\Cxy^{k}$. An analysis of the variance of the RRR estimation (see Sections D.3.4. and D.4 and more specifically Lemma 1 and the proof of Proposition 18 in \cite{kostic2023sharp}) gives for any $\delta \in (0,1)$ w.p.a.l. $1-\delta$ 
\begin{align}
   a_2&\leq \frac{1}{\sqrt{\reg}}\|(\Cx^k+\reg I)^{1/2}( \HTO{k,\reg} - \ETO{k})\|\notag\\
   &\lesssim \frac{1}{n^{1/2}\reg^{(\spar+1)/2}}\ln\delta^{-1} +  
    \frac{1}{\sqrt{\reg\,n}}\left(1+  \frac{\sigma_1(B_k)}{\sigma_{r_k}^2(B_k) -\sigma_{r+1}^2(B_k)} \right)\,\ln\delta^{-1}.
\end{align}
Combining the previous display with \eqref{eq:Bias_dec4}, we get w.p.a.l. $1-\delta$
\begin{align}
    \|\HTO{k} - \ETO{k}\| &\lesssim \reg^{(\alpha-1)/2}   + \frac{1}{n^{1/2}\reg^{(\spar+1)/2}}\ln\delta^{-1}\notag\\
    &\hspace{1cm} + \frac{2\,\sqrt{c_{\RKHS}}}{\sqrt{\reg}}\Vert P_{\leq r_k}^\bot\TO_{k}\Vert  +  
    \frac{1}{\sqrt{\reg\,n}}\left(1+  \frac{\sigma_1(B_k)}{\sigma_{r_k}^2(B_k) -\sigma_{r+1}^2(B_k)} \right)\,\ln\delta^{-1}.
\end{align}

Since we assumed that the spectrum of $\TO_{k}$ decreases exponentially fast to $0$, that is %
$\lambda_{r_k}{\lesssim}-\frac{\alpha\log n}{2(\alpha+\beta)}$, and assuming in addition that 
the gap $gap_{r_k}$ is bounded away from $0$. Then, for $\reg \in (0,1)$ small, the dominating terms in the above display are the first two terms and we propose to balance $\reg$ using only those two. Hence we get for $\reg \asymp n^{-\frac{1}{\rpar+\spar }}$ w.p.a.l. $1-\delta$
\begin{align}
\label{eq:opnorm_bound}
   \|\HTO{k} - \ETO{k}\| & \lesssim n^{-\frac{\rpar-1}{2(\rpar+\spar)}}\ln\delta^{-1}%
    .
\end{align}

\medskip

{\bf Perturbation bounds.}~For simplicity we assume here that the all the eigenvalues admit multiplicity $1$. By a standard Davis-Kahan perturbation argument, we get
\begin{align*}
    |\nu_i - \widehat\nu_i | &\leq \| \lefun_i \|   \|  \refun_i  \| \| \HTO{} - \ETO{} \|\\
    \|\lefun_i - \elefun_i \| &\leq \| \lefun_i \|   \|  \refun_i  \| \frac{\| \HTO{} - \ETO{}  \|}{\gap_i}\\
    \|\refun_i - \erefun_i \| &\leq \| \lefun_i \|   \|  \refun_i  \| \frac{\| \HTO{} - \ETO{}   \|}{\gap_i}
\end{align*}

\noindent
\textbf{Final Bound on the metric}. 
An union combining \eqref{eq:opnorm_bound} for any $k\in [N]$, we get w.p.a.l. $1-\delta$ that the condition in \eqref{eq:relative_error_assump} in Lemma \ref{prop:d_eta_metric} is satisfied with 
$$
\varepsilon_0=\varepsilon_1 = n^{-\frac{\rpar-1}{2(\rpar+\spar)}}\ln(N\delta^{-1}) = : \varepsilon_n(\delta).
$$

Proposition \ref{prop:d_eta_metric} guarantees w.p.a.l. $1-\delta$ that for any $k\in [N]$, the operators $\HTO{k},\ETO{k}$ with corresponding spectral decomposition $(\nu^{(k)}_i,P^{(k)}_i)$ and $(\widehat{\nu}^{(k)}_i,\widehat{P}^{(k)}_i)$: $\forall i\in [r_k]$,
\begin{align}
\label{eq:entriwise_error_cost_matrix2}
    \vert d_{\eta}\big((\nu^{(k)}_i,P^{(k)}_i),(
\widehat{\nu}^{(k)}_i,\widehat{P}^{(k)}_i)\big) \vert \leq 2 \sqrt{2} %
\frac{\| \lefun_i^{(k)} \|   \|  \refun_i^{(k)}  \|}{\gap^{(k)}_i \wedge |\lambda_i^{(k)}|}
\,\varepsilon_n(\delta),\quad \forall i\in [r_k],\;\forall k\in [N]. 
\end{align}

\end{proof}

\subsection{Auxiliary results}

We propose a control on the metric $d_\eta(\cdot,\cdot)$. For simplicity, we assume that all the eigenvalues of the Koopman transfer operators are of multiplicity $1$. 

\begin{prp}
    \label{prop:d_eta_metric}
Let $\varepsilon_0,\varepsilon_1 \in (0,1/2)$ be an absolute constant such that, for any $i\in [r]$,
\begin{equation}
\label{eq:relative_error_assump}
\frac{|\nu_i-\hat\nu_i|}{|\nu_i|} \leq \frac{\| \lefun_i \|   \|  \refun_i  \| }{|\nu_i|}\varepsilon_0 \quad\text{and}\quad  \| P_i - \widehat{P}_i  \| \leq  \frac{\| \lefun_i \|   \|  \refun_i  \|}{\gap_i} \varepsilon_1.
\end{equation}
Then we have for any $i\neq j \in [r]$
\begin{align*}
   & \vert d_\eta(\nu_i,P_i),(\nu_j, P_j))  - d_\eta(\widehat\nu_i,\widehat{P}_i),(\widehat\nu_j, \widehat{P}_j)) \vert \\
   &\hspace{2cm}\leq 2 \sqrt{2} \left( \left(\frac{\| \lefun_i \|   \|  \refun_i  \| }{|\nu_i|} \vee \frac{\| \lefun_j \|   \|  \refun_j  \| }{|\nu_j|}   \right)\varepsilon_0 +  \left(  \frac{\| \lefun_i \|   \|  \refun_i  \|}{\gap_i} \vee \frac{\| \lefun_j \|   \|  \refun_j  \|}{\gap_j} \right)  \, \varepsilon_1   \right) .
\end{align*}
Similarly for any $i\in [r]$
\begin{align*}
   &  d_\eta(\nu_i,P_i),(\widehat{\nu}_i, \widehat{P}_i)) \leq 2 \sqrt{2} \left(\frac{\| \lefun_i \|   \|  \refun_i  \| }{|\nu_i|} \varepsilon_0 +   \frac{\| \lefun_i \|   \|  \refun_i  \|}{\gap_i}  \, \varepsilon_1   \right) .
\end{align*}
\end{prp}

\begin{proof}[Proof of Proposition \ref{prop:d_eta_metric}]
The metric $d_\eta$ is a convex combination of two parts.

We focus on the distance between generator eigenvalues, which is the same as the polar distance $d_{val}$ between transfer operator eigenvalues. Similarly as above, we have by the triangular inequality 
\begin{align*}
  \big| d_{val}(\nu,\nu') - d_{val}(\widehat{\nu},\widehat{\nu}')\big|   \leq \|(\tau,\omega)-(\hat\tau,\hat\omega)\|_2 + \|(\tau',\omega')-(\hat\tau',\hat\omega')\|_2
\end{align*}

Using Lemma \ref{lem:complexnumberinpolarform}, we get that 
\begin{align}
    \|(\tau,\omega)-(\hat\tau,\hat\omega)\|_2^2 &\leq |\nu- \hat\nu|^2 + \arcsin\left( \frac{|\nu- \hat\nu|^2}{4 |\nu| |\hat\nu|}\right).
\end{align}
Assume the relative eigenvalue error is small:
\begin{equation}
\label{eq:relative_error_assump-bis}
\frac{|\nu-\hat\nu|}{|\nu|} \vee \frac{|\nu'-\hat\nu'|}{|\nu'|} \leq \varepsilon <\frac{1}{2}.
\end{equation}
Under \eqref{eq:relative_error_assump-bis} we have $|\hat\nu|\ge (1-\varepsilon)|\nu|$ and therefore
$$
u := \frac{|\nu-\hat\nu|^2}{4|\nu||\hat\nu|} \leq
\frac{\varepsilon^2}{4(1-\varepsilon)} < 1,
$$
so the argument of $\arcsin(\cdot)$ lies in $(0,1)$ as required. Moreover, since $\arcsin(x)$ is Lipschitz near $0$ and $\arcsin(x)\le (\pi/2) x$ for all $x\in[0,u]$, we get
$$
\arcsin\!\Big( \frac{|\nu- \hat\nu|^2}{4 |\nu| |\hat\nu|}\Big)
\le \frac{\pi}{2}\frac{|\nu- \hat\nu|^2}{4 |\nu| |\hat\nu|}.
$$

Hence
\begin{align*}
\|(\tau,\omega)-(\hat\tau,\hat\omega)\|_2^2
&\leq |\nu- \hat\nu|^2 \Big(1 + \frac{\pi}{8|\nu||\hat\nu|}\Big)\leq |\nu- \hat\nu|^2 \Big(1 + \frac{\pi}{8(1-\varepsilon)|\nu|^2}\Big),
\end{align*}
and therefore
\begin{equation}\label{eq:polar_norm_bound}
\|(\tau,\omega)-(\hat\tau,\hat\omega)\|_2
\le |\nu- \hat\nu|\sqrt{1 + \frac{\pi}{8(1-\varepsilon)|\nu|^2}}
\le \sqrt{2}\frac{|\nu- \hat\nu|}{|\nu|},
\end{equation}
since $|\nu|\leq 1$ for all transfer operator eigenvalues. 

Apply the same bound to $(\nu',\hat\nu')$ and combine with the first display to obtain, for each matched pair of eigenvalues,
$$
\big| d_{val}(\nu,\nu') - d_{val}(\widehat{\nu},\widehat{\nu}')\big|
\lesssim \Big(\frac{|\nu- \hat\nu|}{|\nu|} + \frac{|\nu'- \hat\nu'|}{|\nu'|}\Big).
$$
Now apply this inequality to every eigenvalue pairs $i\neq j\in[r]$. In view of \eqref{eq:relative_error_assump}, we get
\begin{equation}\label{eq:final_polar_rate}
\big| d_{val}(\nu_i,\nu_j) - d_{val}(\widehat{\nu}_i,\widehat{\nu}_j)\big|
\lesssim \left( \frac{\| \lefun_i \|   \|  \refun_i  \| }{|\nu_i|} + \frac{\| \lefun_j \|   \|  \refun_j  \| }{|\nu_j|} \right) \varepsilon_0, \quad \forall i\neq j\in [r],
\end{equation}

For the Grassmanian part, we have for any $i\neq j\in [r]$ 
\begin{align*}
\bigg \vert    \| P_i - P_j \|   -  \| \widehat{P}_i - \widehat{P}_j \|   \bigg \vert   &\leq \| P_i - \widehat{P}_i - (P_j - \widehat{P}_j) \|\leq 2\sqrt{2} \left(  \frac{\| \lefun_i \|   \|  \refun_i  \|}{\gap_i} \vee  \frac{\| \lefun_j \|   \|  \refun_j  \|}{\gap_j}  \right) \varepsilon_1.
\end{align*}
Combining the last two displays gives the first result. The second result follows from a similar and actually simpler argument.
\end{proof}

\begin{lemma}
\label{lem:complexnumberinpolarform}
Let \(z_1=r_1e^{i\theta_1}\) and \(z_2=r_2e^{i\theta_2}\) be complex numbers in polar form with \(r_1,r_2\ge 0\) and \(\theta_1,\theta_2\in [0,2\pi)\). Then
\[
|z_1-z_2|^2=(r_1-r_2)^2+2r_1r_2\bigl(1-\cos(\theta_1-\theta_2)\bigr)
= (r_1-r_2)^2 + 4r_1r_2\sin^2\!\Big(\frac{\theta_1-\theta_2}{2}\Big).
\]
\end{lemma}

\begin{proof}
Write the difference and compute its squared modulus:
\[
|z_1-z_2|^2 = |r_1e^{i\theta_1}-r_2e^{i\theta_2}|^2
= (r_1e^{i\theta_1}-r_2e^{i\theta_2})(r_1e^{-i\theta_1}-r_2e^{-i\theta_2}).
\]
Expanding yields
\[
|z_1-z_2|^2 = r_1^2 + r_2^2 - r_1r_2\bigl(e^{i(\theta_1-\theta_2)}+e^{-i(\theta_1-\theta_2)}\bigr).
\]
Using \(e^{i\phi}+e^{-i\phi}=2\cos\phi\) with \(\phi=\theta_1-\theta_2\) gives
\[
|z_1-z_2|^2 = r_1^2 + r_2^2 - 2r_1r_2\cos(\theta_1-\theta_2).
\]
Rearrange the first two terms as a perfect square plus a correction:
\[
r_1^2 + r_2^2 - 2r_1r_2\cos\phi
= (r_1^2 + r_2^2 - 2r_1r_2) + 2r_1r_2(1-\cos\phi)
= (r_1-r_2)^2 + 2r_1r_2(1-\cos\phi).
\]
Finally, apply the trigonometric identity \(1-\cos x = 2\sin^2(x/2)\) to obtain
\[
2r_1r_2(1-\cos\phi) = 4r_1r_2\sin^2\!\Big(\frac{\phi}{2}\Big),
\]
which yields the claimed expression.
\end{proof}

\section{Comparison with other similarity measures}
\label{appendix: comparison_exp}

\subsection{Experiment protocol}
\paragraph{Simulated system and shifts.} We consider a referent linear oscillatory system that is the sum of two simple harmonic oscillators with frequencies 0.5Hz and 1.0Hz, respectively, with a noisy trajectory of length 4001 samples sampled at 200Hz, which is an additive Gaussian noise with standard deviation of $1e-2$. We compare the Koopman operator of the referent system with those of shifted
systems according to four scenarios: 
\begin{enumerate}[label=(\alph*)]
    \item \textbf{Frequency shift}, changes the 1Hz harmonic frequency from  0.6Hz to 2.5Hz in 39 evenly spaced frequencies.
    \item \textbf{Decay rate shift}, changes the 1Hz harmonic decay rate from -0.3 (diverging) to 3.0 (converging) in 67 evenly spaced rates.
    \item \textbf{Subspace shift (rank)} gradually transforms the 1Hz sine
    wave into a 1Hz square wave signal using a Fourier
    Decomposition of a square wave signal with increasing order up to 50. Series formulation of a square wave signal: $s(t) = \frac{4}{\pi} \sum_{n=0}^{\infty} \frac{1}{2n+1} \, \sin\!\big((2n+1)t\big)$.   
    \item \textbf{Sampling frequency shift} where the system is sampled at different sampling frequencies ranging from 100Hz to 300Hz instead of the reference 200Hz. Performed in 19 evenly spaced sampling frequencies.
\end{enumerate}
Koopman operators are estimated from sampled trajectories in each scenario with the RRR method~\citep{kostic2022learning}. We consider the linear kernel, the context (sliding window) is set to one second, the operators' rank is always fixed to twice the number of harmonic oscillators, and the Tikhonov regularization is set to $1e-8$.

\paragraph{Compared similarity measures.} We consider our proposed metric SGOT set with $\eta =0.5$. SOT, an OT-based similarity comparing eigenvalues~\citep{redman2024identifying}. GOT, an OT-based similarity comparing eigensubspaces with a Grassmannian metric and weighted by the normalized eigenvalues ~\citep{antonini2021geometry}. Note that compared to its theoretical definition, we extend the similarity to non-normal operators by taking the absolute value of eigenvalues. We also included the metrics induced by the Hilbert-Schmidt and Operator norms, and the Martin similarity~\citep{martin2002metric}, which compares poles of LDS transfer functions.

\subsection{Ablation study for parameter $\eta$ of SGOT }
\begin{figure}
    \centering
    \includegraphics[width=\linewidth]{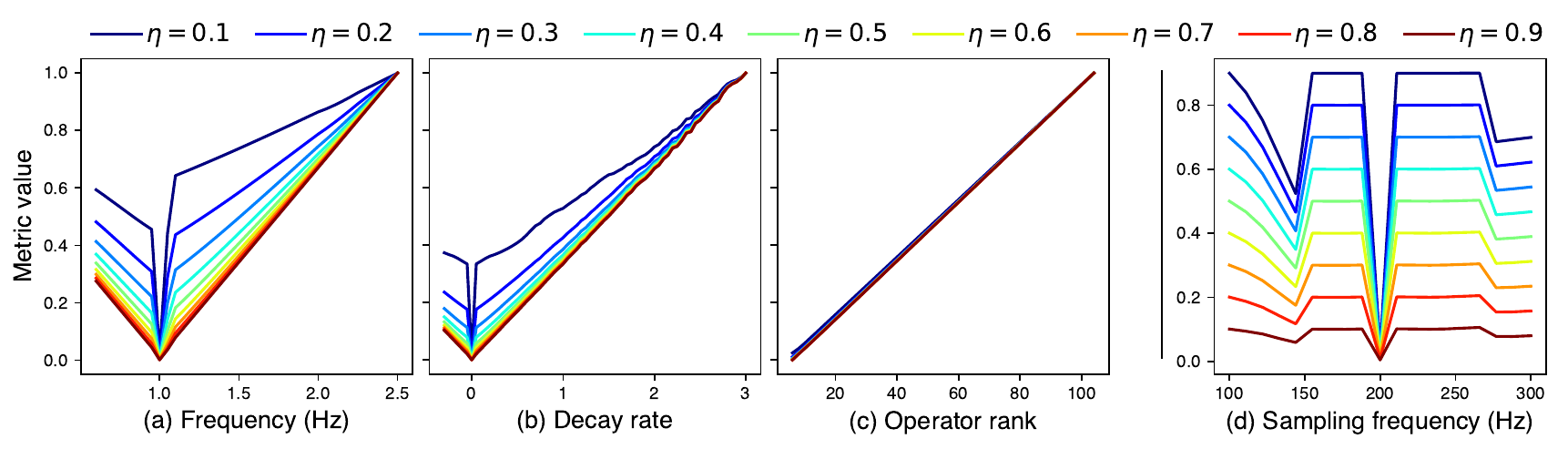}
    \caption{Influence of the $\eta$ parameter in SGOT under four scenarios of shifts of a linear oscillatory system: (a) frequency shift, (b) decay rate shift, (c) operator rank/subspace shift, (d) sampling frequency variation. In scenarios (a,b,c), metric values are normalized by their maximum.}
    \label{fig: ablation study eta}
\end{figure}
Following the same protocol presented in the previous paragraph, we compare our proposed metric SGOT with the parameter controlling the balance between eigenvalues and eigensubspaces $\eta$ ranging in $[0.1,0.2,\ldots,0.9]$. Results are presented in \Cref{fig: ablation study eta}. In scenarios (a,b,c) for any $\eta$, SGOT behaves piecewise linearly, where the ascent gets steeper for scenario (a,b) as $\eta$ decreases (eigensubspaces have more weights in the cost function). For scenario (c), SGOT behaves similarly for all $\eta$. Finally, SGOT becomes slightly more sensitive to the sampling frequency as $\eta$ decreases. In (d), the metric scale is not normalized, and the metric values remain relatively small.

\section{Machine learning on dynamical systems}
\label{appendix: ML experiment}
\subsection{Experimental protocol}
We evaluate similarity performances on a time series classification task. We selected 14 multivariate datasets from the UEA database \citep{ruiz2021great} whose main characteristics are described in \Cref{tab:datasets characteristics}. For each dataset, we estimate operators for individual time series of $n$ samples with the RRR method~\citep{kostic2022learning}, with the linear kernel, a Tikhonov regularization of $1e-2$, an arbitrary sampling frequency $f_{samp} \triangleq \min(100, (n/2)*0.2)$ and a context window $w_{len} \triangleq \min(50, n/2)$. Once all operators are estimated, we perform a 10-iteration Monte-Carlo cross-validation with a 0.7/0.3 train/test split without any preprocessing step. To perform classification, we consider K-Nearest Neighbors (K-NN) estimators defined with similarities: Hilbert-Schmidt, Operator, Martin~\citep{martin2002metric}, SOT~\citep{redman2024identifying}, GOT~\citep{antonini2021geometry}, and our metric SGOT. Note that the initialization-invariant Binet-Cauchy similarity has been excluded from this experiment as it relates to the Martin distance. At each cross-validation iteration, the number of neighbors (K) and the parameter $\eta$ for SGOT metric are set by grid search with a 5-fold cross-validation on the train set. K peaked between 1 and 10 and $\eta \in [0.01,0.1,0.5,0.9,0.99]$. Scores are evaluated in terms of accuracy, and a training time limit has been set to 5 hours per dataset/metric pair. The experiment has been seeded for reproducibility. 

\begin{table}[H]
    \centering
    \caption{Datasets main characteristics: \textit{Size}: number of time series, \textit{Channels}: number of dimensions per time series, \textit{Length}: time series length, \textit{Classes}: number of classes.}
    \label{tab:datasets characteristics}
    \resizebox{0.6\linewidth}{!}{
    \begin{tabular}{l|cccc}
        \toprule
         & \#Size & \#Channels & Length & \#Classes \\
        \midrule
        AtrialFibrillation & 30 & 2 & 640 & 3 \\
        BasicMotions & 80 & 6 & 100 & 4 \\
        Cricket & 180 & 6 & 1197 & 12 \\
        EigenWorms & 259 & 6 & 17984 & 5 \\
        Epilepsy & 275 & 3 & 206 & 4 \\
        ERing & 300 & 4 & 65 & 6 \\
        FingerMovements & 416 & 28 & 50 & 2 \\
        HandMovementDirection & 234 & 10 & 400 & 4 \\
        Handwriting & 1000 & 3 & 152 & 26 \\
        Heartbeat & 409 & 61 & 405 & 2 \\
        NATOPS & 360 & 24 & 51 & 6 \\
        SelfRegulationSCP1 & 561 & 6 & 896 & 2 \\
        StandWalkJump & 27 & 4 & 2500 & 3 \\
        UWaveGestureLibrary & 440 & 3 & 315 & 8 \\
        \bottomrule
    \end{tabular}
    }
\end{table}

\subsection{Additional results}
\label{appendix: ML additional results}

\paragraph{Classification accuracy table.}
In addition to scores comparison plots between our metric SGOT and competitive similarities in the main body (see \Cref{fig: classfication score plot}), \Cref{tab: classification results} provides mean and standard deviation of accuracy scores per dataset and metric computed over the 10 iterations. Our metric SGOT is the best performer on all datasets, followed by GOT, another OT-based metric that only refers to eigensubspaces in its cost function. Also, SOT, a third OT-based similarity comparing operator, only from eigenvalues, performs poorly. Incorporating eigenvalues and eigensubspaces within the cost function improves performance on numerous datasets. By being more conservative (see \Cref{fig: ablation study}), Hilbert-Schmidt and Operator underperform compared to SGOT. Note that the Operator norm times out on Heartbeat. Lastly, Martin distance performs poorly and fails on some datasets due to its ill-definedness in some settings.

\begin{table}[H]
    \centering
    \caption{Classification accuracy scores. Datasets on rows and similarities on columns. \textbf{Best} and \underline{second best} performers are highlighted. Accuracy scores are denoted: $<mean>\pm<std>$.}
    \label{tab: classification results}
    \resizebox{\linewidth}{!}{%
    \begin{tabular}{l|cccccc}
\toprule
 & Hilbert-Schmidt & Operator & Martin & SOT & GOT & SGOT \\
\midrule
AtrialFibrillation & 0.31 $\pm$ 0.07 & 0.32 $\pm$ 0.13 & 0.27 $\pm$ 0.09 & 0.24 $\pm$ 0.14 & \underline{0.4 $\pm$ 0.12} & \textbf{0.44 $\pm$ 0.13} \\
BasicMotions & 0.48 $\pm$ 0.15 & 0.51 $\pm$ 0.13 & 0.3 $\pm$ 0.06 & 0.35 $\pm$ 0.1 & \underline{0.8 $\pm$ 0.07} & \textbf{0.93 $\pm$ 0.05} \\
Cricket & 0.33 $\pm$ 0.05 & 0.28 $\pm$ 0.05 & 0.07 $\pm$ 0.03 & 0.11 $\pm$ 0.04 & \underline{0.63 $\pm$ 0.04} & \textbf{0.85 $\pm$ 0.05} \\
ERing & 0.79 $\pm$ 0.04 & 0.74 $\pm$ 0.05 & 0.15 $\pm$ 0.04 & 0.39 $\pm$ 0.04 & \underline{0.85 $\pm$ 0.01} & \textbf{0.87 $\pm$ 0.03} \\
EigenWorms & 0.6 $\pm$ 0.04 & 0.57 $\pm$ 0.04 & $\emptyset$ & 0.57 $\pm$ 0.06 & \underline{0.71 $\pm$ 0.04} & \textbf{0.88 $\pm$ 0.03} \\
Epilepsy & 0.46 $\pm$ 0.05 & 0.52 $\pm$ 0.06 & $\emptyset$ & 0.34 $\pm$ 0.04 & \underline{0.78 $\pm$ 0.04} & \textbf{0.93 $\pm$ 0.03} \\
FingerMovements & 0.51 $\pm$ 0.05 & \underline{0.54 $\pm$ 0.03} & $\emptyset$ & 0.51 $\pm$ 0.05 & 0.51 $\pm$ 0.05 & \textbf{0.57 $\pm$ 0.03} \\
HandMovementDirection & 0.23 $\pm$ 0.04 & 0.23 $\pm$ 0.03 & \underline{0.27 $\pm$ 0.04} & 0.21 $\pm$ 0.05 & 0.24 $\pm$ 0.05 & \textbf{0.29 $\pm$ 0.03} \\
Handwriting & 0.12 $\pm$ 0.02 & 0.12 $\pm$ 0.02 & 0.05 $\pm$ 0.01 & 0.05 $\pm$ 0.01 & \underline{0.21 $\pm$ 0.02} & \textbf{0.42 $\pm$ 0.02} \\
Heartbeat & 0.7 $\pm$ 0.04 & $\emptyset$ & \underline{0.71 $\pm$ 0.04} & 0.69 $\pm$ 0.02 & 0.7 $\pm$ 0.04 & \textbf{0.73 $\pm$ 0.04} \\
NATOPS & \underline{0.76 $\pm$ 0.04} & 0.73 $\pm$ 0.05 & 0.25 $\pm$ 0.04 & 0.41 $\pm$ 0.04 & 0.74 $\pm$ 0.04 & \textbf{0.77 $\pm$ 0.04} \\
SelfRegulationSCP1 & 0.57 $\pm$ 0.02 & 0.56 $\pm$ 0.03 & $\emptyset$ & \underline{0.57 $\pm$ 0.05} & 0.56 $\pm$ 0.03 & \textbf{0.61 $\pm$ 0.02} \\
StandWalkJump & \underline{0.5 $\pm$ 0.15} & 0.41 $\pm$ 0.13 & $\emptyset$ & 0.39 $\pm$ 0.13 & 0.3 $\pm$ 0.13 & \textbf{0.69 $\pm$ 0.16} \\
UWaveGestureLibrary & 0.24 $\pm$ 0.04 & 0.21 $\pm$ 0.05 & $\emptyset$ & 0.13 $\pm$ 0.02 & \underline{0.47 $\pm$ 0.03} & \textbf{0.64 $\pm$ 0.04} \\
\midrule
avg. rank (lower is better) & 3.11 $\pm$ 1.03 & 3.75 $\pm$ 1.13 & 5.17 $\pm$ 1.43 & 4.39 $\pm$ 1.23 & \underline{2.57 $\pm$ 1.18} & \textbf{1.27 $\pm$ 0.73} \\
\bottomrule
\end{tabular}
    }
\end{table}

\paragraph{Computation times.} 
During the classification experiment, we kept track of all metric computation time, which we average per metric in \Cref{tab: compute time}. Operator norm is the least efficient metric, followed by the Hilbert-Schmidt. The most efficient similarities are Martin and SOT; however, they performed poorly. SGOT and GOT are slightly less effective than Martin and SOT but much more efficient than Hilbert-Schmidt and Operator.

\begin{table}[H]
    \caption{Average computation time per similarity on all validation folds.}
    \label{tab: compute time}
    \centering
    \begin{tabular}{cccccc}
        \toprule
        Hilbert-Schmidt & Operator & Martin & SOT & GOT & SGOT \\
        \midrule
        4.96ms & 13.04ms & 0.02ms & 0.03ms & 0.14ms & 0.12ms \\
        \bottomrule
    \end{tabular}
\end{table}

\paragraph{Critical diagram difference.}
Considering results from all 10 iterations of the Monte Carlo cross-validation, we compute the critical diagram difference to statistically compare all metric performances based on rank. The diagram is depicted in \Cref{fig:critical diagram difference}. The test significance level is set to 0.05. We use Friedman's test to reject the null hypothesis (All metrics' performances are similar) and compute the critical differences using the Nemenyi post-hoc test. Results show that SGOT is the best performer and statistically different from the second performer (SGOT). 

\begin{figure}[H]
    \centering
    \includegraphics[width=0.8\linewidth]{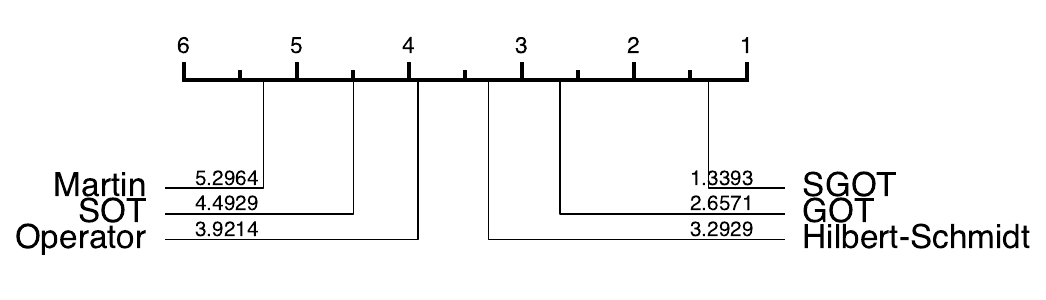}
    \caption{Critical diagram difference for comparing metrics' performances on a classification task. The classifiers are K-NN defined with the metrics: \textit{Hilbert-Schmidt, Operator, Martin, SOT, GOT, and SGOT (ours)}. Computed from the performance of all 10 iterations of the Monte Carlo cross-validation. The test significance level is set to 0.05, the null hypothesis is rejected with Friedman's test, and the critical differences are computed using Nemenyi post-hoc test.}
    \label{fig:critical diagram difference}
\end{figure}

\paragraph{2D T-SNE embeddings.}
We illustrate the dimensionality reductions capabilities of the different similarity measures. We selected 5 datasets from fields including human activity recognition, motion recognition, and biomedical applications. For the 5 selected datasets and all similarities, dataset samples are embedded as a 2D vector with the T-distributed Stochastic Neighbor Embedding (T-SNE) \cite{maaten2008visualizing} method fitted on the cross-distance matrix estimated with the similarities: Hilbert-Schmidt, Operator, Martin, SOT, GOT, and SGOT. \Cref{fig:all tsne} displays the embeddings for all 5 datasets and metrics. On the Eigenworms and Epilepsy datasets, Martin is ill-defined, and the similarity values cannot be computed. No clusters or classes can be identified for Hilbert-Schmidt, Operator, Martin, and SOT. Regarding other OT-based metrics, GOT better identifies classes; however, they do not form distinct clusters as obtained with our metric SGOT. 

\begin{figure}[H]
    \centering
    \includegraphics[width=1.0\linewidth]{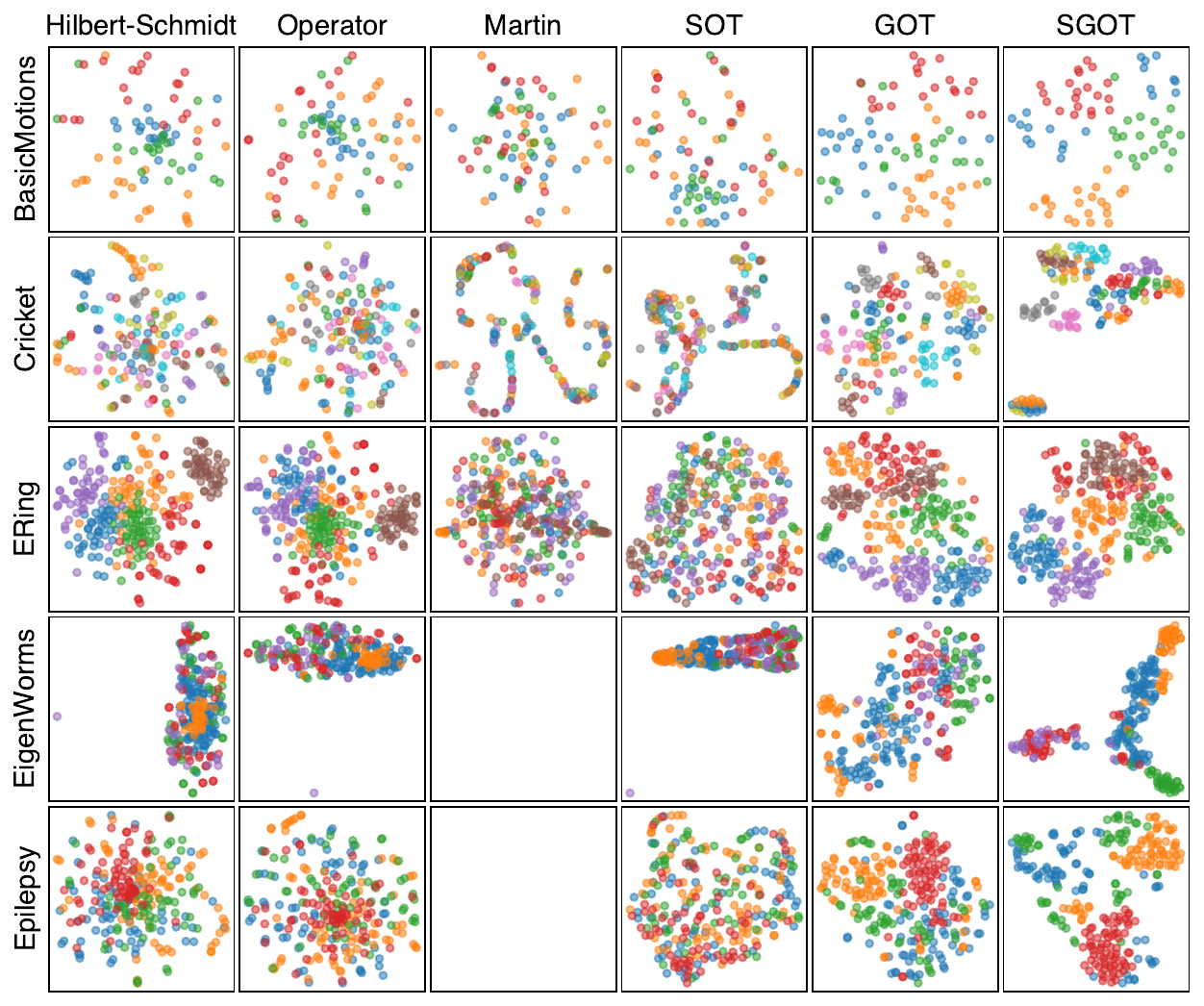}
    \caption{T-SNE 2D-embeddings of the classification datasets: \textit{BasicMotions, Cricket, Ering, EigenWorms, Epilepsy} based on similarities: \textit{Hilbert-Schmidt, Operator, Martin, SOT, GOT and SGOT (ours)}. Each point represents a dataset sample (a time series) whose color corresponds to its class. The Martin similarity is ill-defined on \textit{EigenWorms} and \textit{Epilepsy} datasets; thus, the corresponding T-SNEs are missing.  }
    \label{fig:all tsne}
\end{figure}

\section{Barycenter of dynamical systems}
\label{appendix: barycenter experiment}

\subsection{Interpolation between 1D dynamical systems}

\paragraph{Experimental settings.}
In this experiment, we compare the interpolation between dynamical systems through weighted Fréchet barycenters of their Koopman operators, estimated with a linear kernel, for different metrics. The two systems are linear oscillatory systems, each being the sum of two simple harmonic oscillators with different frequencies and decay rates, and additive Gaussian noise. The first system $\mbT^{(0)}$ combines a convergent low frequency oscillator ($\omega = 1.7\text{Hz},~\rho=-0.2$, amplitude=1.0) with a divergent high frequency oscillator ($\omega = 4.7\text{Hz},~\rho=0.2$, amplitude=0.2) . The second system $ \mbT^{(1)}$ is reversed; it combines a divergent low frequency oscillator ($\omega = 0.7\text{Hz},~ \rho=0.2$, amplitude=1)  with a convergent high frequency oscillator ($\omega = 11.3\text{Hz},~ \rho=-0.2$, amplitude=1). Both systems are noisy with a Gaussian noise with variance $\sigma^2 = 1e-4$. The systems Koopman operators are estimated with the RRR methods~\citep{kostic2022learning} from trajectories of length 5000 samples at 800Hz. RRR estimator is set to estimate a rank 4 operator with context window of 400 samples, a linear kernel, and Tikhonov regularization of $1e-8$. The interpolation is controlled by a ratio parameter $\gamma$ going from 0 to 1 in 0.1 steps. At each interpolation step, the weights in the Fréchet mean (see \eqref{eq:frechet_mean_problem}) are ($1-\gamma, \gamma$). We compare (a) the Hilbert-Schmidt metric without spectral decomposition constraints given by $\mbT_{bar} = (1-\gamma) \mbT^{(0)} + \gamma \mbT^{(1)}$, (b) the Hilbert-Schmidt metric with spectral decomposition constraints, and (c) our proposed metric SGOT. For (b) and (c), the barycentric operators are estimated with the proposed optimization scheme described in \cref{appendix: barycenter}. In both cases,  the initialization of the barycenter corresponds to the average of eigenvalues and eigenfunctions.
For the Hilbert-Schmidt (b), the barycenter optimizer is set with a $3e-5$ learning rate, a maximal number of iterations of 2000, with 1 gradient descent per coordinate at each iteration, the stopping criteria corresponds to a consecutive metric error lower than $1e-6$.
For the SGOT (c), $\eta = 0.9$, and the barycenter optimizer is set with a $1e-2$ learning rate, a maximal number of iterations of 200, with 10 gradient descent per coordinate at each iteration, the stopping criteria corresponds to a consecutive metric error lower than $1e-6$.
Finally, for displaying the predicted signals from the interpolated barycenter in \Cref{fig: interpolation}, all predictions started with the same initialization set, being the first 400 samples of a linear system that is the sum of 4 harmonic oscillators of the systems to interpolate. 

\paragraph{Additional results.}
For all constrained Hilbert-Schmidt (b) and SGOT (c) interpolated barycenters, we kept track of the decay rate and frequency of the two associated harmonic oscillators, the loss values, and the computation time. \Cref{fig:convergence} displays the normalized losses decrease per gradient descent step for each metric and interpolation step. The representation is in gradient descent step as the number of iterations and gradient descent step per cycle differ from metric to metric. In \Cref{fig: decay freq interpolation} we display the decays and frequencies of the interpolated barycenters. In particular, \Cref{fig:convergence} shows that the barycenter algorithm has converged for any metric and interpolation step. However, \Cref{fig: decay freq interpolation} shows that constrained Hilbert-barycenter (b) remains stuck in a local minima close to the initialization. In contrast, the SGOT barycenter perfectly (linearly) interpolates the decay and frequency between the source and target systems. Furthermore, the average computation time per gradient descent step is 13.11ms for the constrained Hilbert-Schmidt, while being 2.29ms for our metric SGOT, meaning that the barycenter algorithm is approximately 6x faster with the SGOT metric compared to the Hilbert-Schmidt.

\begin{figure}
    \centering
    \includegraphics[width=0.4\linewidth]{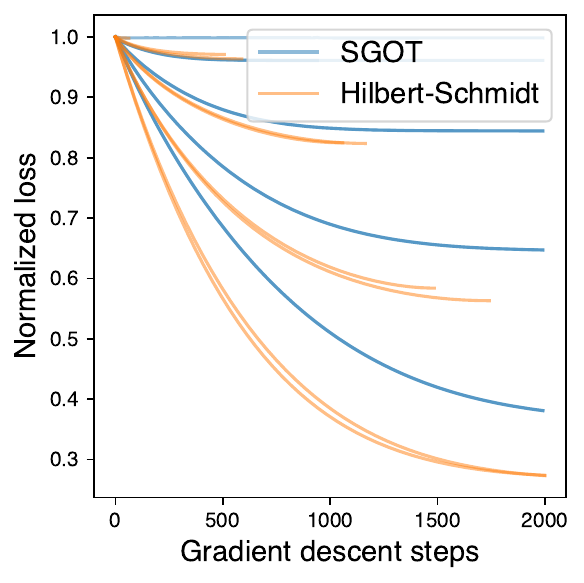}
    \caption{Normalized loss value per gradient descent step for the constrained Hilbert-Schmidt (b) and SGOT (c) barycenter for any interpolation step.}
    \label{fig:convergence}
\end{figure}

\begin{figure}
    \centering
    \includegraphics[width=0.7\linewidth]{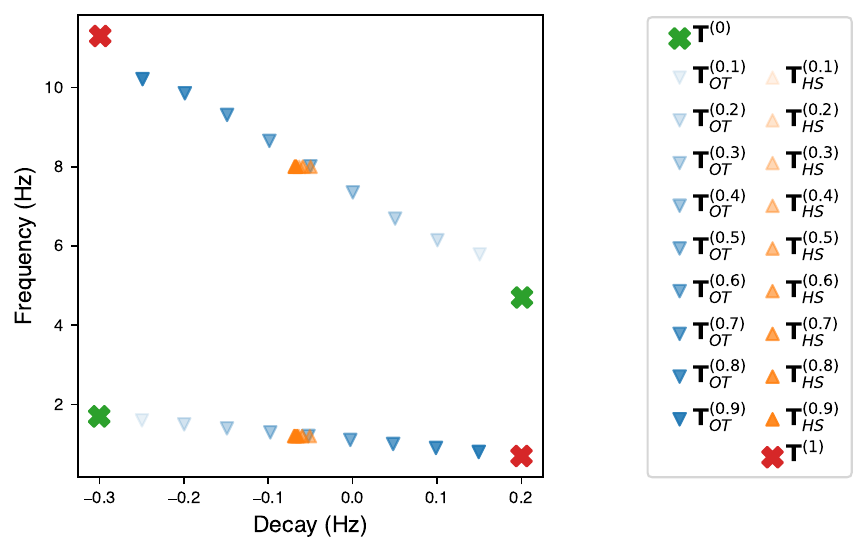}
    \caption{Decay rates and frequencies of the two harmonic oscillators associated with the interpolated barycenters for the constrained Hilbert-Schmidt (b) and the SGOT (c) metrics. The source system harmonic oscillator is in red, and the target system harmonic oscillator is in green.}
    \label{fig: decay freq interpolation}
\end{figure}

\subsection{Fluid dynamic interpolation}
\paragraph{Experimental settings.}
We aim to compute the barycenter of two fluid dynamics systems. To that end,
we consider the \textit{Flow past a bluff object}
dataset~\citep{tali2025flowbench}, which gathers trajectories of time-varying 2D
velocity and pressure fields of incompressible Navier-Stokes fluids flowing
around static objects. We select two trajectories, one with a cylinder object (Huggingface dataset file: harmonic/93) and the other with a triangular object (Huggingface dataset file: skeleton/48). For each trajectory, we only kept the
velocity field along the flow direction, leading to trajectories containing
242 samples of 1024x256 grids, which we down-sampled to grids with a 256x64
resolution. We estimate a Koopman operator with a linear kernel using the
RRR method from each trajectory sampled at 100Hz with a context window of 1, and a Tikhonov regularization of 1. The operators are restricted to the fourth leading
eigenvalues and eigenfunctions. We compute the SGOT barycenter with the optimization scheme described in \Cref{appendix: barycenter} with an initialization
being the average of eigenvalues and eigenfunctions. For the SGOT (c), $\eta = 0.01$, and the barycenter optimizer is set with a $1e-4$ learning rate, a maximal number of iterations of 100, with 10 gradient descent per coordinate at each iteration, the stopping criteria corresponds to a consecutive metric error lower than $1e-6$.

\end{document}